\documentclass[review]{elsarticle}

\journal{Artificial Intelligence Journal}

\usepackage{graphicx}
\usepackage{placeins}
\usepackage{subfig} 
\usepackage{epstopdf}
\usepackage{tikz}
\usepackage{pgfplots}
\pgfplotsset{compat=1.11}

\usepackage{algorithm}
\usepackage[noend]{algpseudocode}

\usepackage{multirow}

\usepackage{amsfonts}
\usepackage{amsmath}
\usepackage{amsthm}
\usepackage{mathtools}
\usepackage{bm}
\usepackage{bbm}
\usepackage{thmtools} 
\usepackage{thm-restate}
\usepackage{rotating}

\usepackage[multiple]{footmisc}

\usepackage{eurosym}

\usepackage{tabulary}
\usepackage{booktabs}
\usepackage{amssymb}
\usepackage{placeins}
\usepackage{tabu}

\newcommand{\alg}{\textsf{AdComb}}

\newcommand{\bi}{x}
\newcommand{\bu}{y}

\newcommand{\realbi}{\hat{x}}
\newcommand{\realbu}{\hat{y}}

\usepackage{color}

\allowdisplaybreaks

\begin{document}

\begin{frontmatter}

\title{Online Joint Bid/Daily Budget Optimization of Internet Advertising Campaigns}
\tnotetext[mytitlenote]{A preliminary version is published in~\cite{nuara2018combinatorial}. The novel contributions provided by this extended version are: a theoretical study of the properties of our algorithms and a more extensive experimental evaluation.}

\author{Alessandro Nuara and Francesco Trov\`o and Nicola Gatti and Marcello Restelli\\
Dipartimento di Elettronica, Informazione e Bioingegneria\\
Politecnico di Milano, Milano, Italy\\
\{alessandro.nuara, francesco1.trovo, nicola.gatti, marcello.restelli\}@polimi.it
}

\begin{abstract}
Pay-per-click advertising includes various formats (\emph{e.g.}, search, contextual, social) with a total investment  of more than $200$ billion USD per year worldwide.
An advertiser is given a daily budget to allocate over several, even thousands, campaigns, mainly distinguishing for the ad, target, or channel.
Furthermore, publishers choose the ads to display and how to allocate them employing auctioning mechanisms, in which every day the advertisers set for each campaign a bid corresponding to the maximum amount of money per click they are willing to pay  and the fraction of the daily budget to invest. 
In this paper, we study the problem of automating the online joint bid/daily budget optimization of pay-per-click advertising campaigns over multiple channels.
We formulate our problem as a \emph{combinatorial semi-bandit problem}, which requires solving a special case of the Multiple-Choice Knapsack problem every day.
Furthermore, for every campaign, we capture the dependency of the number of clicks on the bid and daily budget by \emph{Gaussian Processes}, thus requiring mild assumptions on the regularity of these functions. 
We design four algorithms and show that they suffer from a regret that is upper bounded with high probability as $O(\sqrt{T})$, where $T$ is the time horizon of the learning process.
We experimentally evaluate our algorithms with synthetic settings generated from real data from Yahoo!, and we present the results of the adoption of our algorithms in a real-world application with a daily average spent of 1,000 Euros  for more than one year.
\end{abstract}


\end{frontmatter}

\section{Introduction}

Online advertising has been given wide attention by the scientific community as well as by industry from more than two decades.
In the first half of $2019$, in the US, the spent on \emph{search} advertising alone was more than $28$ billion USD~\cite{iab2018iab}, which represents about $50\%$ of the \emph{total} online advertising market. 
The development of techniques to automate the Internet advertising market is crucial both for the publishers and the advertisers, and artificial intelligence can play a prominent role in this scenario. 
In the present paper, we focus on \emph{pay-per-click} advertising---including different formats, \emph{e.g.}, \emph{search}, \emph{contextual}, \emph{social}---in which an advertiser pays only once a user has clicked her ad.

An advertiser is usually given a daily budget to allocate over several, even thousands, campaigns, distinguishing for the ad (\emph{e.g.}, including text content and/or images), target (\emph{e.g.}, keyword, geographical area, language, interests), or channel (\emph{e.g.}, Google, Facebook, Bing).
In pay-per-click advertising, to get an ad impressed, the advertisers take part in an auction carried out by the publisher or the advertising manager, in which they set a bid and  a daily budget for each campaign~\cite{qin2015sponsored}.
The bid represents the maximum amount of money the advertisers are willing to pay for a single click, whereas the daily budget is the maximum spent in a day for a campaign.
The advertisers' goal is to choose the values of the bid and daily budget for every single campaign to maximize the revenue subject to a cumulative budget constraint over all the campaigns.
This optimization problem is particularly challenging, as it includes many intricate subproblems.
For instance, the auctioning mechanisms are not \emph{truthful}, meaning that the best bid for an advertiser may be different from the actual value per click.
As a result, an advertiser needs to resort to learning tools to estimate the revenue provided by the values of the bid when searching for the best bidding strategy.
More precisely, the \emph{Generalized Second Price auction} (GSP), commonly used for search advertising, is well-known not to be  truthful even without budget constraints~\cite{king2007internet}.
Instead, the \emph{Vicrey-Clarke-Groves} mechanism (VCG), commonly used for contextual and social advertising, is truthful only in the unrealistic case in which there is no daily budget constraint~\cite{varian2014vcg,gatti2015truthful}.~\footnote{
Notably, the value per click is usually unknown at the setup of an advertising campaign, and its estimation may require a long time. 
Thus, even without daily budget constraints, an advertiser could not have enough information to bid truthfully in the VCG mechanism.}
Furthermore, every learning algorithm can collect at most a sample per day for every campaign, thus raising severe data scarcity issues.
Besides, the optimization of the daily budget is a combinatorial problem that needs to be solved every day in a short time.

\subsection{Previous Results on Bid and/or Daily Budget Optimimzation}

A few works in the algorithmic economic literature tackle the campaign advertising problem by combining learning and optimization techniques.
The joint  bid/daily budget optimization problem is studied in~\cite{zhang2012joint} and~\cite{kong2018combinational}, where the authors provide two offline learning approaches. 
These two works are characterized by models with a huge number of parameters whose estimates require a considerable amount of data.~\footnote{Some of these parameters (\emph{e.g.}, the position of the ad for every display and click) cannot be observed by advertisers, not allowing the employment of those approaches in practice.}
%
%
%
In~\cite{thomaidou2014toward}, the authors separate the optimization of the bid from that one of the daily budget and use a genetic algorithm to optimize the budget and subsequently applying some bidding strategies.
While the works in~\cite{zhang2012joint,kong2018combinational,thomaidou2014toward} provide no theoretical guarantees, some theoretical results are known for the convergence of some bidding strategies in a single-campaign scenario without budget constraints~\cite{markakis2010discrete}.
%
%
%
%
%
%
%

Online learning approaches with regret guarantees are known only for the restricted cases with a single campaign and a budget constraint over all the length of the campaign without temporal deadlines.~\footnote{
This last assumption rarely holds in real-world applications where, instead, results are expected by a given deadline.}
In particular, in~\cite{ding2013multi} and~\cite{xia2015thompson}, the authors work on a finite number of bid values and exploit a multi-armed bandit approach, whereas the approach proposed in~\cite{trovo2016budgeted} deals with a continuous space of values for the bid and shows that assuring worst-case guarantees leads to the worsening of the average-case performance.

We also mention some works dealing with daily budget optimization~\cite{xu2015smart,italia2017internet,li2018efficient, amin2012budget},  bidding strategies aiming to maximize the profit~\cite{ren2017bidding} and the number of conversions under a budget constraint~\cite{wang2016display, zhang2014optimal,lee2013real,wu2018budget}, cost-per-click~\cite{yang2019bid,weinar2016feedback} and cost-per-action constraints~\cite{ kong2018demystifying}.
%
%
%
Finally, some works study the attribution problem of conversions in display advertising~\cite{geyik2014multitouch,kireyev2016display}.

\subsection{Original Contributions}
We formulate the joint bid/daily budget optimization problem as a combinatorial semi bandit problem~\cite{chen2013combinatorial},
in which, at every round and for each campaign, an advertiser chooses a pair of bid/daily budget values and observes some information on the performance of the campaign.
We discretize the bid/daily budget space, and we formulate the optimization problem as a special case of the Multiple-Choice Knapsack problem~\cite{sinha1979multiple}, that we solve  by dynamic programming in a fashion similar to the approximation scheme for the knapsack problem. 
%
%
We resort to Gaussian Process (GP) regression models~\cite{rasmussen2006gaussian} to estimate the uncertain parameters of the optimization problem (\emph{e.g.}, number of clicks and value per click), as the adoption of GPs requires mild assumptions on the regularity of the functions we need to learn and allows one to capture the correlation among the data thus mitigating the data scarcity issue.
%
%
%

We design four bandit techniques to balance exploration and exploitation in the learning process, that return either samples or upper confidence bounds of the stochastic variables estimated by the GPs to use in the optimization problem.
%
%
We show that our algorithms suffer from a regret that is upper bounded with high probability as $O(\sqrt{T})$, where $T$ is the learning time horizon.

Finally, we experimentally evaluate the convergence of our algorithms to the optimal (clairvoyant) solution and its empirical regret as the size of the problem varies using a realistic simulator based on the \emph{Yahoo!} Webscope $A3$ dataset.
Furthermore, we present the results of the adoption of our algorithms in a real-world application over a period longer than one year, with an average cumulative daily budget of about 1,000 Euros.

\subsection{Structure of the Paper}

The paper is structured as follows. In Section~\ref{sec:problemmformulation}, we formulate our problem as a combinatorial semi bandit. In Section~\ref{sec:method}, we describe our algorithms, whereas, in Section~\ref{sec:analysis}, we provide their theoretical regret analysis. In Section~\ref{sec:experiments}, we present the experimental evaluation of the algorithms. Section~\ref{sec:conclusions} concludes the paper and provides a discussion on future works. In~\ref{sex::appendix}, we report the proofs of our theoretical results. 

\section{Problem Formulation} \label{sec:problemmformulation}

This section is structured as follows.
In Section~\ref{subsec:model}, we introduce the combinatorial optimization problem an advertiser needs to tackle every day to find the best pairs of bid/daily budget values.
In Section~\ref{subsec:banditformmulation}, we formulate the learning problem as a combinatorial semi bandit.
In Section~\ref{subsec:banditstateoftheart}, we describe the main results related to our problem known in the machine learning literature.

\subsection{Optimization Problem Formulation}
\label{subsec:model}

An advertiser is provided with a collection of $N \in \mathbb{N}$ advertising campaigns $\mathcal{C} = \{ C_1, \ldots, C_N \}$, where $C_j$ is the $j$-th campaign, a finite time horizon of $T \in \mathbb{N}$ days, and a \emph{spending plan} $\mathcal{B} = \{\overline{\bu}_1, \ldots, \overline{\bu}_T\}$, where $\overline{\bu}_t \in \mathbb{R}^+$ is the cumulative budget an advertiser is willing to spend at day $t \in \{1, \ldots, T\}$ over all the campaigns.~\footnote{
We assume that the set of campaigns $\mathcal{C}$ and spending plan $\mathcal{B}$ are given.
The campaigns distinguish for the ad, channel, and target.
In real-world applications, the set of campaigns and the spending plan can be optimized from data, \emph{e.g.}, by setting up campaigns with specific targets and adopting a different cumulative daily budget for every day~\cite{gasparini2018targeting}.}\footnote{
For the sake of presentation, from now on, we set one day as the unitary temporal step of our algorithms.
The application of our techniques to different time units is straightforward by opportunely scaling the variables and parameters.
%
}
For every day $t \in \{1, \ldots, T\}$ and every campaign $C_j, j \in \{1, \ldots, N \}$, an advertiser chooses a bid/daily budget pair $\bm{a}_{j,t} = (\bi_{j,t}, \bu_{j,t})$.
The bid $\bi_{j,t}$ takes values in a finite space $X \subset \mathbb{R}^+$ and is constrained to be in the interval $[\underline{\bi}_{j,t}, \overline{\bi}_{j,t}]$, where $\underline{\bi}_{j,t}$ and $\overline{\bi}_{j,t} \in \mathbb{R}^+$ are the minimum and maximum bid an advertiser can choose, respectively.
Similarly, the daily budget $\bu_{j,t}$ takes values in a finite and, for simplicity, evenly-spaced set $Y \subset \mathbb{R}^+$ and is constrained to be in $[\underline{\bu}_{j,t}, \overline{\bu}_{j,t}]$, where $\underline{\bu}_{j,t}$ and $\overline{\bu}_{j,t} \in \mathbb{R}^+$ are the minimum and maximum daily budget an advertiser can set, respectively.~\footnote{
The platforms allow different discretization for the bid and the daily budget. 
For instance, on the Facebook platform, the daily budget discretization has a step of $1.00$ Euro, while, on the Google Adwords platform, a step of $0.01$ Euro is allowed. }
By choosing a specific bid/daily budget pair $(\bi_{j,t}, \bu_{j,t})$ for a day $t$ on campaign $C_j$, an advertiser gets an expected revenue of $v_j \ n_j(\bi_{j,t}, \bu_{j,t})$, where $v_j$ is the value per click for campaign $C_j$ and $n_j(\bi_{j,t}, \bu_{j,t})$ is the corresponding expected number of clicks.
The goal of an advertiser at every day $t \in \{1, \ldots, T\}$ is the choice of the values of the bid and daily budget for every campaign to maximize the cumulative expected revenue.
These values can be found by solving the following optimization problem.
\begin{subequations}
\begin{align}
	\max_{(\bi_{j,t}, \bu_{j,t}) \in X \times Y} & \sum_{j=1}^N v_j \ n_j(\bi_{j,t}, \bu_{j,t}) \label{formulation:objectivefunction} \\
	\text{s.t.} & \sum_{j=1}^N \bu_{j,t} \leq \overline{\bu}_t \label{formulation:budgetconstraints} \\
	& \underline{\bi}_{j,t} \leq \bi_{j,t} \leq \overline{\bi}_{j,t}& \forall j \label{formulation:boxconstraints1} \\
	& \underline{\bu}_{j,t} \leq \bu_{j,t} \leq \overline{\bu}_{j,t} & \forall j \label{formulation:boxconstraints2} 
\end{align}
\end{subequations}
The objective function stated in Equation~\eqref{formulation:objectivefunction} is the weighted sum of the expected number of clicks $n_j$ generated by all the campaigns, where the weights $v_j$ are the campaigns' value per click.
The constraint in Equation~\eqref{formulation:budgetconstraints} is a budget constraint, forcing one not to spend more than the cumulative daily budget limit, while the constraints in Equations~\eqref{formulation:boxconstraints1} and~\eqref{formulation:boxconstraints2} force the variables to assume values the given ranges for bid and daily budget.
The above problem is a special case of the Multiple-Choice Knapsack (MCK)~\cite{kellerer2004multiple}, which is a variant of the knapsack problem where the items are divided into classes and at most one item per class can be chosen. 
In the problem above, a class corresponds to a campaign, whereas an item corresponds to a pair of bid/daily budget values. 
%

Each function $n_j(\bi, \bu)$ is fully described by $|X| \,\, |Y|$ parameters.
However, exploiting the structure of the problem, a much more concise representation using only $2\,|X|$ parameters can be provided.~\footnote{
The reduction of the number of the parameters from $|X|\,\, |Y|$ to $2\,|X|$ does not affect the complexity of the optimization problem, but it plays a crucial role when one needs to learn these parameters.}
We factorize the dependency of the number of clicks on $\bi$ and $\bu$ as follows:
\begin{equation} \label{eq:nmax}
	n_j(\bi, \bu) := \min \left\{ n_{j}^{\mathsf{sat}}(\bi), \bu \ e^{\mathsf{sat}}_{j}(\bi) \right\},
\end{equation}
where the functions $n_{j}^{\mathsf{sat}}(\bi)$ and $e^{\mathsf{sat}}_{j}(\bi)$ describe:
\begin{itemize}
	\item the maximum number of clicks $n_{j}^{\mathsf{sat}} : X \rightarrow \mathbb{R}^+$ that can be obtained with a given bid $\bi$ without any daily budget constraint (or, equivalently, letting $\bu \rightarrow +\infty$);
	\item the number of clicks per unit of daily budget $e_{j}^{\mathsf{sat}} : X \rightarrow \mathbb{R}^+$ with a given bid $\bi$, under the assumption that the number of clicks depends linearly on the daily budget when $n_j \leq n_{j}^{\mathsf{sat}}$.
\end{itemize}
The rationale is that the maximum number of clicks $n_{j}^{\mathsf{sat}}(\bi)$ obtainable with bid $\bi$ is finite and depends on the number of auctions an advertiser can win when using bid $\bi$.
More specifically, $n_{j}^{\mathsf{sat}}(\bi)$ is monotonically increasing in $\bi$, since the number of auctions won by an advertiser and the average quality of the slots in which the ad is displayed monotonically increase in $\bi$. 
Notably, the cost per click monotonically increases in $\bi$ and, therefore, $e^{\mathsf{sat}}_{j}(\bi)$ monotonically reduces in $\bi$.
Finally, fixed the value of $\bi$, the number of clicks increases linearly in the daily budget $\bu$, where the slope is the number of clicks per unit of daily budget $e_{j}^{\mathsf{sat}}(\bi)$, until the maximum number of clicks $n_{j}^{\mathsf{sat}}(\bi)$ obtainable with bid $\bi$ is achieved.
For the sake of clarity, we report in Figure~\ref{fig:curvaesempioclick} an example of function $n_j$ for four values of bid; the black dashed curve depicts $\max_{\bi \in X} \{n_j(\bi, \bu)\}$ as the daily budget $\bu$ varies.

\begin{figure}[ht!]
	\centering
	\includegraphics[width = .8 \columnwidth]{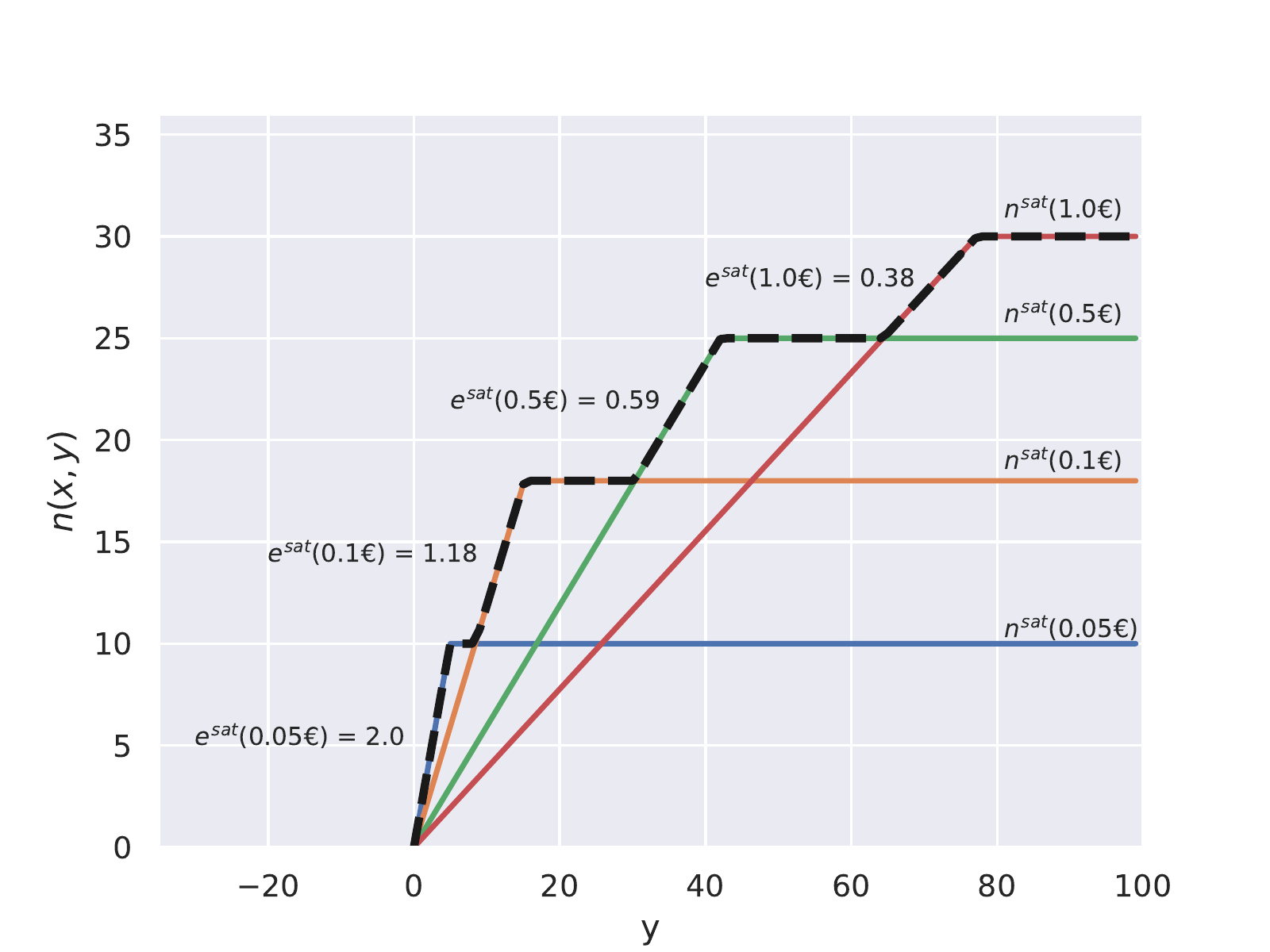}
	\caption{An example of $n_j$ for four values of bid.}
	\label{fig:curvaesempioclick}
\end{figure}

\subsection{Learning Problem Formulation}
\label{subsec:banditformmulation}

In concrete scenarios, the functions $n_j(\cdot, \cdot)$ and the parameters $v_j$ in the optimization problem stated in Equations~\eqref{formulation:objectivefunction}--\eqref{formulation:boxconstraints2} are not \emph{a priori} known, but they need to be estimated online.
Thus, an algorithm needs to gather as much information as possible about these functions during the operational life of the system, and, at the same time, not to lose too much revenue in exploring suboptimal bid/daily budget allocations (a.k.a.~exploration/exploitation dilemma).
Thus, our learning problem can be naturally formulated in a sequential decision fashion~\cite{cesa2006prediction} as a Combinatorial Semi Bandit problem (CSB)~\cite{chen2013combinatorial}.~\footnote{
Another approach to solve this problem is to use a multistage method, \emph{e.g.}, backward induction, but, even for problems with few campaigns and for only $2$ stages, such technique would require a huge computational effort that makes these methods an unfeasible solution in practice.}
In the CSB framework, at every round, the learner chooses from a finite set of options, called \emph{arms}, a subset of these, called \emph{superarm}, subject to some combinatorial constraints (\emph{e.g.}, knapsack constraints). 
Subsequently, the learner observes the payoffs of every single arm belonging to the chosen superarm and gets the corresponding reward.
In the optimization problem stated in Equations~\eqref{formulation:objectivefunction}--\eqref{formulation:boxconstraints2}, each arm corresponds to a bid/daily budget pair, each superarm corresponds to a collection of pairs, one per campaign, and the constraints consist in satisfying the overall daily budget constraint and the range constraints for the bid and daily budget.
The payoff of every arm is the revenue we obtain by setting a bid/daily budget pair.

We denote with $\mathcal{D} = X \times Y$, where $|\mathcal{D}| = M$, the finite space of bid/daily budget pairs. 
The learning process proceeds as follows.
Every day $t$, an advertiser, called \emph{learner} from hereafter, chooses a superarm $\mathcal{S} \in \mathcal{D}^N$, where $S_t := ( \bm{a}_{1,t}, \ldots, \bm{a}_{N,t} )$ and the arm $\bm{a}_{j,t} \in \mathcal{D}$ is the bid/daily budget pair we set for  campaign $C_j$ at day $t$.
Such a superarm $S_t$ must be feasible according to the constraints in Equations~\eqref{formulation:budgetconstraints}-\eqref{formulation:boxconstraints2}.
The choice of superarm $S_t$ leads to a revenue expressed in terms of clicks and value per click.
We denote the random variable corresponding to the number of clicks of campaign $C_j$ by $N_j(\bi_{j,t}, \bu_{j,t})$ and the random valuable corresponding to the value per click of campaign $C_j$ by $V_j$.
Thus, the revenue is a random variable $\sum_j V_j N_j(\bi_{j,t}, \bu_{j,t})$.
We denote with $r_{\bm{\mu}}(S_t)$ the expected value of the revenue when we pull superarm $S_t$, and the vector $\boldsymbol \mu$ of the expected revenues of each arm of every campaign is:
\begin{equation*}
	\bm{\mu} := (v_1 n_1(\bi_1, \bu_1), \ldots, v_1 n_1(\bi_{M}, \bu_{M}), \ldots, v_N n_N(\bi_1, \bu_1), \ldots, v_N n_N(\bi_{M}, \bu_{M}) ).
\end{equation*}
%
%
From now on, we refer to the problem defined above as the Advertisement Bid/ daily Budget Allocation (ABBA) problem.

A policy $\mathfrak{U}$ solving our problem is an algorithm returning, each day $t \in \{1, \ldots, T\}$, a superarm $S_t$.
Given a policy $\mathfrak{U}$, we define the expected \emph{pseudo-regret} over a time horizon of $T$ as:
\begin{equation*}
	\mathcal{R}_T(\mathfrak{U}) := T \,r_{\bm{\mu}}^* - \mathbb{E} \left[ \sum_{t = 1}^T r_{\bm{\mu}}(S_t) \right],
\end{equation*}
where $r_{\bm{\mu}}^* := r_{\bm{\mu}}(S^*)$ is the expected value of the revenue provided by the clairvoyant algorithm choosing the optimal superarm $S^* = (\bm{a}^*_1, \ldots \bm{a}^*_N \}_{j=1}^N$ that is the solution to the problem in Equations~\eqref{formulation:objectivefunction}-\eqref{formulation:boxconstraints2}, and the expectation $\mathbb{E}[\cdot]$ is taken with respect to the stochasticity of the policy $\mathfrak{U}$.
Our goal is the design of algorithms minimizing the pseudo-regret $\mathcal{R}_T(\mathfrak{U})$.
A recap of the notation defined in this section and used from now on is provided in Table~\ref{tab:notation}.

\begin{table}[t!]
\small
\caption{Notation}
\label{tab:notation}
\begin{center}
	\begin{tabular}{r c p{8cm} }
		\toprule
		$\mathcal{C}$ & $\triangleq$ & set of advertising campaigns\\
		$N$ & $\triangleq$ & number of campaigns\\
		${C_j}$ & $\triangleq$ & $j$-th campaign\\
		$t$ & $\triangleq$ & current day\\
		$T$ & $\triangleq$ & time horizon\\
		$\mathcal{B}$ & $\triangleq$ & spending plan\\
		$x_{j,t}$ & $\triangleq$ & set of bid values for the $j$-th campaign at time $t$\\
		$y_{j,t}$ & $\triangleq$ & set of the daily budget values for the $j$-th campaign at time $t$\\
		$\bm{a}_{j,t} = (x_{j,t},y_{j,t})$ & $\triangleq$ & set of the bid/daily budget pairs for the $j$-th campaign at time $t$\\
		$\mathcal{D}$ & $\triangleq$ & set of the possible bid/daily budget pairs for each campaign\\
		$X$ & $\triangleq$ & space of the possible bid values\\
		$Y$ & $\triangleq$ & space of the possible daily budget values\\
		$M$ & $\triangleq$ & cardinality of $\mathcal{D}$\\
		$S_{t}$ & $\triangleq$ & tuple of bid/daily budget pairs (superarms) for the campaigns at time $t$\\
		$v_j$ & $\triangleq$ & value per click of the $j$-th campaign\\
		$n_j(\bi_{j,t}, \bu_{j,t})$ & $\triangleq$ & expected number of clicks given by a bid/daily budget pair $(\bi_{j,t}, \bu_{j,t})$\\
		$\bm{\mu}$ & $\triangleq$ & vector of the expected revenues of each arm of every campaign\\
		$\bm{a}_{j}^*$ & $\triangleq$ & optimal bid/budget pair for the $j$-th campaign\\
		$S^*$ & $\triangleq$ & optimal superarm for the set of campaigns $\mathcal{C}$\\
		$r_{\bm{\mu}}(S_t)$  & $\triangleq$ & expected value of the revenue when we pull superarm $S_t$ at round $t$\\
		$r_{\bm{\mu}}^* = r_{\bm{\mu}}(S^*)$ & $\triangleq$ & expected value of the revenue when we pull the optimal superarm $S_t$ at round $t$\\
		\bottomrule
	\end{tabular}
\end{center}
\end{table}

\subsection{Previous Results on Related Learning Problems}
\label{subsec:banditstateoftheart}

The Combinatorial Bandit framework is introduced in the seminal work~\cite{chen2013combinatorial}, in which the authors also propose an algorithm based on statistical upper confidence bounds, namely CUCB.
Under the assumption that the support of the payoff functions is bounded on $[0,1]$, the CUCB algorithm provides a regret $O(\log(T))$.
The CUCB does not exploit the potential correlation existing among the expected reward of the arms, which makes its application to our specific scenario unfeasible due to the long time needed for learning the parameters.
Another work in the Combinatorial Bandit literature related to our paper is the one presented in~\cite{sankararaman2018combinatorial}, in which the authors design an algorithm for a combinatorial semi-bandit problem with knapsack constraints.
Differently from our setting, in this scenario, every single arm is assigned a specific budget that recedes every time the arm is pulled. 
The process stops as soon as one of the arms runs out of its budget. 
Differently, in our setting, we have a cumulative budget for every day, allowing the pull of the arms for an arbitrary number of times.
In the Combinatorial Bandit literature, few works are known to exploit arm correlation to speed up the learning process.
The most significant result, provided by Degenne and Perchet~\cite{degenne2016combinatorial}, describes an algorithm for the specific case of combinatorial constraints in which we are allowed to pull a fixed number of arms at each round.

Other works related to ours can be found in the (non-combinatorial) MAB literature.
More precisely, Srinivas \emph{et al.}~\cite{srinivas2010gaussian} propose the GP-UCB algorithm that employs GPs in the basic stochastic MAB setting where, at every round, only a single arm can be pulled.
The pseudo-regret of the GP-UCB algorithm is proved to be upper bounded with high probability as $\tilde{O}(\sqrt{T})$.
These algorithms cannot be directly applied to our scenario where, instead, we can pull a superarm subject to a set of constraints.
Notably, we extend the work in~\cite{srinivas2010gaussian} to the  more challenging combinatorial setting and show that it has the same upper bound on the pseudo-regret $\tilde{O}(\sqrt{T})$.

\section{Proposed Method} \label{sec:method}

In Section~\ref{sec:mainalgorithm}, we describe our main algorithm.
Subsequently, in Sections~\ref{sec:updatesubroutine}, \ref{sec:samplingsubroutine}, and~\ref{sec:optimmizesubroutine}, we provide the details of the subroutines used by the main algorithm.
%

\subsection{The Main Algorithm}
\label{sec:mainalgorithm}

Algorithm~\ref{alg:bbo}, named \textsf{AdComB}, provides the high-level pseudocode of our method, and its scheme is reported in Figure~\ref{fig:architecture}.
The input to the algorithm is composed of: the discrete set of bid values $X$, the discrete set of daily budget values $Y$, a model $\mathcal{M}^{(0)}_j$ that, for each campaign $C_j$, captures the prior knowledge of the learner about the function $n_j(\cdot, \cdot)$ and the parameter $v_j$, a spending plan $\mathcal{B}$, and a time horizon $T$.
We distinguish three phases that are repeated every day $t \in \{1. \ldots, T \}$.

\begin{algorithm}[th!]
\caption{\textsf{AdComB}}
\label{alg:bbo}
\begin{algorithmic}[1]
\State \textbf{Input}: set $X$ of bid values, set $Y$ of daily budget values, prior model $\{ \mathcal{M}_j^{(0)} \}_{j=1}^N$, spending plan $\mathcal{B}$, time horizon $T$
\For{ $t \in \{1, \ldots, T\}$ }															\label{line:for1}
	\For{ $j \in \{1, \ldots N\}$ }															\label{line:for2}
		\If{$t=1$}																	\label{line:if}
			\State $\mathcal{M}_j \gets \mathcal{M}^{(0)}_j$ 								\label{line:init} 
		\Else
			\State $\textsf{Get} \left( \tilde{n}_{j,t-1}, \tilde{c}_{j,t-1}, \tilde{g}_{j,t-1}, \tilde{v}_{j,t-1} \right)$ 		\label{line:get}
			\State $\mathcal{M}_j \leftarrow \textsf{Update}\left(\mathcal{M}_j, \left( \realbi_{j,t-1}, \realbu_{j,t-1}, \tilde{n}_{j,t-1}, \tilde{c}_{j,t-1},\tilde{g}_{j,t-1}, \tilde{v}_{j,t-1} \right) \right)$ 		\label{line:update}
		\EndIf		\label{line:endif}
		\State $\left( \hat{n}_j(\cdot,\cdot), \hat{v}_j \right)  \leftarrow \textsf{Sampling} \left( \mathcal{M}_j, X, Y \right)$ 			\label{line:bandit}
	\EndFor
	\State $\{\left( \realbi_{j,t}, \realbu_{j,t} \right)\}_{j = 1}^N \leftarrow \textsf{Optimize} \left(\{ \hat{n}_j(\cdot,\cdot), \hat{v}_j,  X, Y \}_{j=1}^N, \overline{y}_t \right)$ 	\label{line:opti}
	\State $\textsf{Pull} \left( \realbi_{1,t}, \realbu_{1,t}, \ldots, \realbi_{N,t}, \realbu_{N,t} \right)$ \label{line:set}
\EndFor
\end{algorithmic}
\end{algorithm}

\begin{figure}[ht!]
\centering
\includegraphics[width = 1.0 \columnwidth]{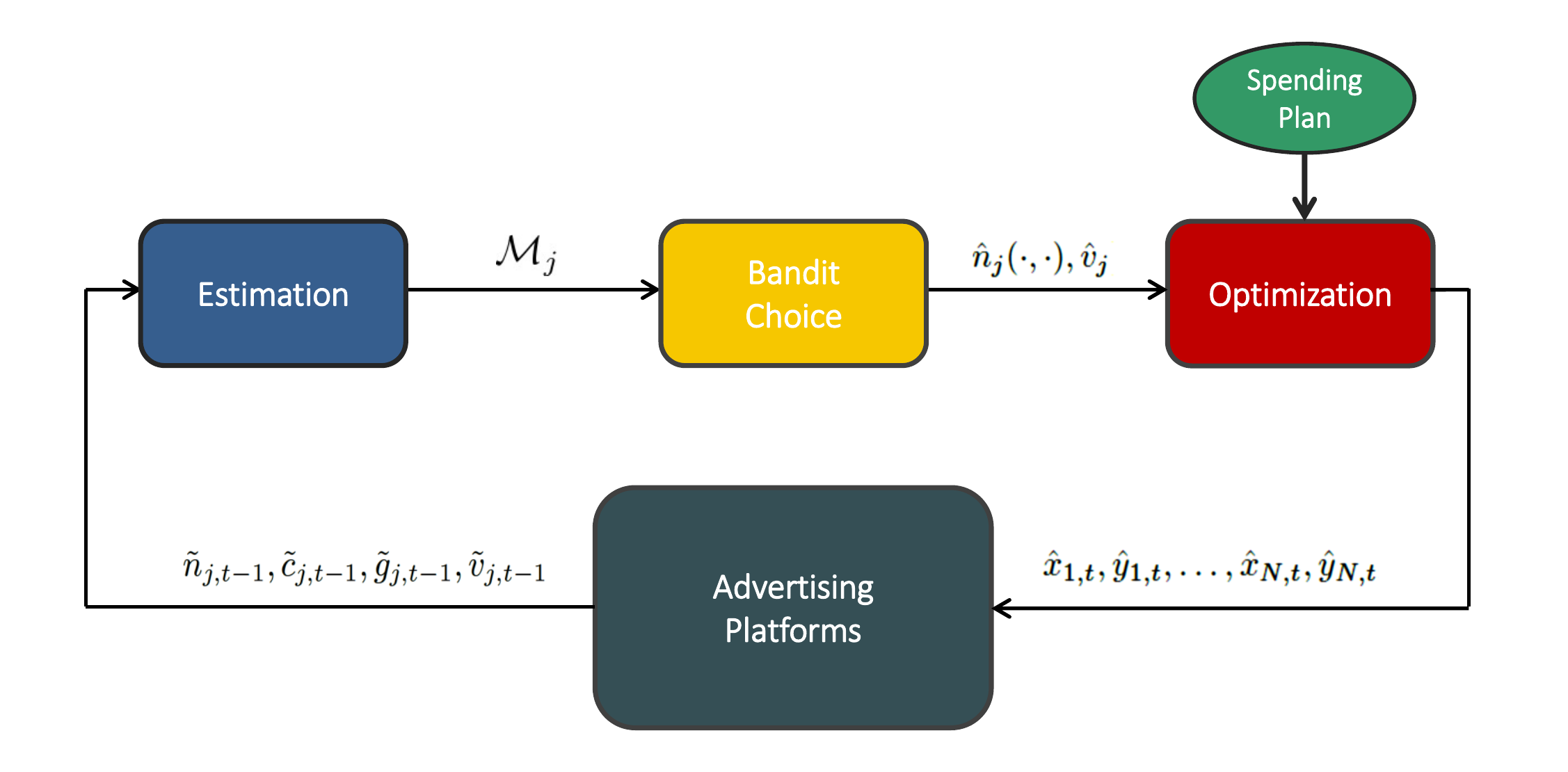}
\caption{The information flow in the \textsf{AdComB} algorithm along the three phases.}
\label{fig:architecture}
\end{figure}

In the first phase (Lines \ref{line:if}--\ref{line:endif}), denoted with \emph{Estimation} in Figure~\ref{fig:architecture}, the algorithm learns, from the observations of days $\{1, \ldots, t-1 \}$, the model $\mathcal{M}_j$ of every campaign $C_j$.
In particular, the model $\mathcal{M}_j$ provides a probability distribution over the average number of clicks $n_j(\bi, \bu)$ as the bid $\bi$ and the daily budget $\bu$ vary and over the average value per click $v_j$.
The first day the algorithm is executed, no observation is available, and, thus, the model $\mathcal{M}_j$ is based on the prior $\mathcal{M}^{(0)}_j$ (Line~\ref{line:init}).
Conversely, during the subsequent days, for every campaign $C_j$, the algorithm gets an observation corresponding to day $t-1$ (Line~\ref{line:get}) composed of:
\begin{itemize}
	\item ($\tilde{n}_{j,t-1}$) the actual number of clicks;
	\item ($\tilde{c}_{j,t-1}$) the actual total daily cost of the campaign;
	\item ($\tilde{g}_{j,t-1}$) the time when the daily budget $\overline{\bu}_{t-1}$ exhausted at $t-1$, if so;
	\item ($\tilde{v}_{j,t-1}$) the actual value per click;
\end{itemize}
and, subsequently, updates the model of each campaign $\mathcal{M}_j$ using those observations (Line~\ref{line:update}).

In the second phase (Line~\ref{line:bandit}), denoted with \emph{Bandit Choice} in Figure~\ref{fig:architecture}, the algorithm uses the model $\mathcal{M}_j$ just updated to infer an estimate of the function $n_j(\cdot, \cdot)$,  for every  value of bid and daily budget in $X$ and $Y$, respectively, and of the parameter $v_j$.
We denote these estimates with $\hat{n}_j(\cdot,\cdot)$ and $\hat{v}_j$, respectively.

In the third phase (Lines~\ref{line:opti}--\ref{line:set}), denoted with \emph{Optimization} in Figure~\ref{fig:architecture}, the algorithm employs the estimates $\hat{n}_j (\cdot, \cdot)$ and $\hat{v}_j$ in place of $n_j (\cdot, \cdot)$ and $v_j$ in the problem stated in Equations~\eqref{formulation:objectivefunction}--\eqref{formulation:boxconstraints2}.
Finally, it solves the optimization problem returning the bid/daily budget allocation for the next day $t$ (Line~\ref{line:set}).
%

In what follows, we provide a detailed description of the model $\mathcal{M}_j$ and of the subroutines $\textsf{Update}(\cdot)$, $\textsf{Sampling}(\cdot)$, and $\textsf{Optimize}(\cdot)$ used in Algorithm~\ref{alg:bbo}.

\subsection{Model and \textsf{Update} Subroutine}
\label{sec:updatesubroutine}

As mentioned before, the \textsf{Update} subroutine generates an estimate of the number of clicks $n_j(\cdot,\cdot)$ and value per click $v_j$ using the previous observations.
To avoid data scarcity issues and speed up the learning process, we make mild assumptions on the function $n_j(\cdot, \cdot)$, and we model it by resorting to GPs~\cite{rasmussen2006gaussian}.
These models, developed in the statistical learning field, capture the correlation of the nearby points in the input space exploiting kernel functions.
Moreover, they provide a probability distribution over the output space---in our case the number of clicks---for each point of the input space---in our case the discretized space of bid/daily budget pairs---, thus giving information both on the expected values of the quantities to estimate and their uncertainty.

For the sake of presentation, we describe how we model the maximum number of clicks $n_{j}^{\mathsf{sat}}(\cdot)$ with a GP regression model.
The model directly applies to the number of clicks per unit of daily budget $e_{j}^{\mathsf{sat}}(\cdot)$.
Furthermore, in some situations, the factorization introduced in Equation~\eqref{eq:nmax} may not be exploited by a learning algorithm, as we discuss below.
In these cases, one can adopt a $2$-dimensional GP to model $n_j(\cdot, \cdot)$.
The treatment of this case, called \emph{unfactorized} hereafter, is analogous to that of $n_{j}^{\mathsf{sat}}(\cdot)$, but, every time the factorized model can be employed, its use is preferable due to the curse of dimensionality~\cite{bishop2006pattern}.
In the following, we use \alg-\textsf{F} to refer to the algorithm when the factorized model is used, while we use \alg-\textsf{U} for the case in which we do not use the factorized model.

We model $n_{j}^{\mathsf{sat}}(\cdot)$ for campaign $C_j$ with a GP over the bid space $X$, \emph{i.e.}, using a collection of random variables having a joint Gaussian distribution.
Following the definition provided in~\cite{rasmussen2006gaussian}, a GP is completely specified by the mean $m : X \rightarrow \mathbb{R}$ and covariance $k : X \times X \rightarrow [0, 1]$ functions.
Hence, we denote the GP that models the maximum number of clicks in $C_j$ as follows:
\begin{equation*}
	n_{j}^{\mathsf{sat}}(\bi) := \mathcal{GP} \left( m(\bi), k(\bi, \cdot) \right), \forall \bi \in X.
\end{equation*}
More specifically, the correlation structure we use is given by a squared exponential kernel:
\begin{equation*}
	k(\bi, \bi') = \exp \left\{- \frac{(\bi - \bi')^2}{2\, l^2} \right\} \ \ \forall \bi, \bi' \in X,
\end{equation*}
where $l \in \mathbb{R^+}$ is a length-scale parameter determining the smoothness of the function.~\footnote{If available, \emph{a priori} information on the process can be employed to design a function $m(\bi)$ over the input space $X$ which specifies the mean value. For instance, when information on the maximum number $\theta$ of clicks achievable for any bid is available, one may use a linearly increasing function over the bid space as $m(\bi) = \frac{\bi \theta}{\max_{\bi' \in X} \bi'}$. If instead no \emph{a priori} information is available, one can use a uninformative prior mean by setting $m(\bi) = 0, \forall \bi \in X$.} 
Other common choices for the kernel can be found in~\cite{srinivas2010gaussian}.
%

According to GP model, at every day $t$, the predictive distribution corresponding to the maximum number of clicks $n_{j, t}^{\mathsf{sat}}(\bi)$ on campaign $C_j$ for the bid $\bi$ is estimated by $\mathcal{N}(\hat{\mu}_{j,t-1}(\bi), \hat{\sigma}^2_{j,t-1}(\bi))$ with:
\begin{align*}
	&\hat{\mu}_{j,t-1}(\bi) = m(\bi) + \mathbf{k}(\bi, \boldsymbol{\realbi}_{j,t-1})^\top \Phi^{-1} \left( \tilde{\boldsymbol{n}}_{j,t-1}^{\mathsf{sat}} - \mathbf{m}_{j,t-1} \right),\\
	&\hat{\sigma}^2_{j,t-1}(\bi) = \ k(\bi, \bi) - \mathbf{k}(\bi, \boldsymbol{\realbi}_{j,t-1})^\top \ \Phi^{-1} \ \mathbf{k}(\bi, \boldsymbol{\realbi}_{j,t-1}),
\end{align*}
where $\boldsymbol{\realbi}_{j, t-1}: = \left( \realbi_{j, 1}, \ldots, \realbi_{j, t-1} \right)^\top$ is the vector of the bid set so far, $\mathbf{k}(\bi, \boldsymbol{\realbi}_{j,t-1}) := (k(\bi, \realbi_{j, 1}), \ldots, k(\bi, \realbi_{j, t-1}) )^\top$ is the correlation value for the bid $\bi$ w.r.t.~each element of the vector $\boldsymbol{\realbi}_{j, t-1}$, $\mathbf{m}_{j,t-1} := \left( m(\realbi_{j, 1}), \ldots, m(\realbi_{j, t-1}) \right)^\top$ is the vector of the prior for the input in $\boldsymbol{\realbi}_{j, t-1}$, $\tilde{\boldsymbol{n}}^{\mathsf{sat}}_{j, t-1} := \left( \tilde{n}_{j,  1}^{\mathsf{sat}}, \ldots, \tilde{n}_{j,  t-1}^{\mathsf{sat}} \right)^\top$ is the vector of maximum number of clicks achieved the previous days, $[\Phi]_{h,k} := k(\realbi_{j,h}, \realbi_{j,k}) + \lambda$ is the Gram matrix built on the available data, and $\lambda$ is the variance of the realizations we use in the estimation process.~\footnote{
From now on, we denote with $\mathcal{N}(\mu, \sigma^2)$ the Gaussian with mean $\mu$ and variance~$\sigma^2$.}\footnote{
The computation cost of the estimation can be dramatically reduced by using an alternative, but much more involved, approach whereby the inverse of the Gram matrix $\Phi^{-1}$ is stored and updated iteratively at each day; see~\cite{rasmussen2006gaussian} for details.}
Note that the distribution of the maximum number of clicks at the first day is $\mathcal{N}(m(\bi), k(\bi,\bi))$ for each $\bi \in X$ since no information from the data can be used yet.

The parameter $\tilde{n}_{j, t}^{\mathsf{sat}}$ is set equal to the observation $\tilde{c}_{j, t}$ when the daily budget $\hat{y}_{j,t}$ used for campaign $C_j$ did not exhaust.
When instead $\hat{y}_{j,t}$ exhausted, we have not a direct observation of $\tilde{n}_{j, t}^{\mathsf{sat}}$, and, thus, we set $\tilde{n}_{j, t}^{\mathsf{sat}}$ as a function of the time $\tilde{g}_{j, t}$.
For instance, if we assume a uniform distribution of the clicks over the day, the value of $\tilde{n}_{j, t}^{\mathsf{sat}}$ has the following expression:
\begin{equation*}
	\tilde{n}_{j, t}^{\mathsf{sat}} := \frac{24}{\tilde{g}_{j, t}} \tilde{n}_{j, t},
\end{equation*}
where $\tilde{g}_{j, t} \in (0, 24]$ is expressed in hours.
In general, this relationship can be estimated from historical data coming from past advertising campaigns of products belonging to the same category (\emph{e.g.}, toys, insurances, beauty products).
Conversely, if no information on how the clicks distribute over the day is available, one has to rely on the unfactorized model for $n_j(\cdot, \cdot)$.
Similar considerations hold for the estimation of the value per click $v_j$. 


We estimate $v_j$, at day $t$, from the observations $\tilde{\boldsymbol{v}}_{j, t-1} := (\tilde{v}_{j,1}, \ldots, \tilde{v}_{j,t-1})^\top$ of the previous days up to $t-1$.
We use a single Gaussian probability distribution to model the value per click $v_j$, thus, at every day $t$, given the observations $\tilde{\boldsymbol{v}}_{j, t-1}$, we estimate its mean $\hat{\nu}_{j,t}$ and variance $\hat{\psi}^2_{j,t}$ relying on the Bayesian update of a prior $\mathcal{N}(0, \psi_j^2)$~\cite{gelman2013bayesian}, as follows:
\begin{align*}
	&\hat{\nu}_{j,t-1} := \frac{\psi_j^2 \sum_{h=1}^{t-1} \tilde{v}_{j,h}}{\xi + (t-1)\psi_j^2},\\
	&\hat{\psi}^2_{j,t-1} := \frac{\psi_j^2 \xi}{\xi + (t-1)\psi_j^2},
\end{align*}
where $\xi$ is the measurement noise variance.
To summarize, the data needed for updating the model $\mathcal{M}_j$ corresponding to campaign $C_j$ at day $t$ consists of the following elements:
\begin{itemize}
	\item the values per click $\tilde{\boldsymbol{v}}_{j, t-1}$,
	\item the chosen bids $\boldsymbol{\realbi}_{j,t-1}$,
	\item the maximum number of achievable clicks $\tilde{\boldsymbol{n}}_{j, t-1}^{\mathsf{sat}}$,
	\item the number of clicks per unit of daily budget $\tilde{\boldsymbol{e}}_{j, t-1}^{\mathsf{sat}} := \left( \frac{\tilde{n}_{j,1}}{\tilde{c}_{j,1}}, \ldots, \frac{\tilde{n}_{j,t-1}}{\tilde{c}_{j,t-1}} \right)^\top$.
\end{itemize}

\subsection{\textsf{Sampling} Subroutine}
\label{sec:samplingsubroutine}

The \textsf{Sampling} subroutine aims at returning an estimate of the expected number of clicks and the value per click to use in the optimization problem stated in Equations~\eqref{formulation:objectivefunction}--\eqref{formulation:boxconstraints2}. 
The n\"aive choice of using the expected value computed from $\mathcal{M}_j$ may not provide any guarantee to minimize the regret $\mathcal{R}_T(\mathfrak{U})$, as it is well known in the bandit literature.
To guarantee that our algorithm minimizes the cumulative expected regret, we compute an estimation exploiting the information on the uncertainty provided by model $\mathcal{M}_j$.
More precisely, the model $\mathcal{M}_j$ associated with campaign $C_j$ provides a probability distribution over the values of the function $n_j(\cdot,\cdot)$ and the values $v_j$ can assume. 
This is equivalent to say that $\mathcal{M}_j$ provides a probability distribution over the possible instances of the optimization problem in Equations~\eqref{formulation:objectivefunction}--\eqref{formulation:boxconstraints2}. 
The \textsf{Sampling} subroutine generates, from $\mathcal{M}_j$, a single instance of the optimization problem, assigning a value to $n_j(\bi, \bu)$ for every $\bi \in X, \bu \in Y$ and $v_j$.

We propose two different approaches for the sampling phase, namely \textsf{AdComB-UCB} and \textsf{AdComB-TS}, taking inspiration from the GPUCB algorithm~\cite{srinivas2010gaussian}, and the Thompson Sampling (TS) algorithm~\cite{thompson1933likelihood}, respectively.~\footnote{
For the sake of clarity, in what follows we describe the our sampling procedure for the \alg-\textsf{F} version of \alg{}; the case for the unfactored model \alg-\textsf{U} is analogous.}

The \textsf{AdComB-UCB} algorithm uses upper confidence bounds on the expected value of the posterior distributions to estimate $n_{j}^{\mathsf{sat}}(\bi)$ and  $e_{j}^{\mathsf{sat}}(\bi)$.
More specifically, $n_{j}^{\mathsf{sat}}(\bi)$ and  $e_{j}^{\mathsf{sat}}(\bi)$ are replaced in the optimization problem defined in Equation~\eqref{formulation:objectivefunction}--\eqref{formulation:boxconstraints2} by:
\begin{align*}
&u^{(n)}_{j,t-1}(\bi) := \hat{\mu}_{j,t-1}(\bi) + \sqrt{b^{(n)}_{j,t-1}} \,\hat{\sigma}_{j,t-1}(\bi),\\
&u^{(e)}_{j,t-1}(\bi) := \hat{\eta}_{j,t-1}(\bi) + \sqrt{b^{(e)}_{j,t-1}} \,\hat{s}_{j,t-1}(\bi),
\end{align*}
respectively, where $\hat{\eta}_{j,t-1}(\bi)$ and $\hat{s}^2_{j,t-1}(\bi)$ are the mean and the variance provided by the GP modeling $e_{j}^{\mathsf{sat}}(\cdot)$, respectively, $b^{(n)}_{j,t-1} \in \mathbb{R}^+$ and $b^{(e)}_{j,t-1} \in \mathbb{R}^+$ are non-negative sequences of values, which will be discussed later on in Section~\ref{sec:analysis}.
Properly setting the values of $b^{(n)}_{i,t-1}$ and $b^{(e)}_{i,t-1}$ leads us to design optimistic bounds, that are necessary for the convergence of the algorithm to the optimal solution.
Similarly, for the value per click $v_j$, we use:
\begin{equation*}
	u^{(v)}_{j,t-1} := \hat{\nu}_{j,t-1} + \sqrt{b^{(v)}_{j,t-1}} \,\hat{\psi}_{j,t-1},
\end{equation*}
where $b^{(v)}_{j,t-1} \in \mathbb{R}^+$ is a non-negative sequence of values.

Conversely, the \textsf{AdComB-TS} algorithm draws samples from the distributions corresponding to $n_{j}^{\mathsf{sat}}(\bi)$ and $e_{j}^{\mathsf{sat}}(\bi)$ and, consequently, computes the value of $n_j(\bi, \bu)$ to be used in the following optimization phase.
More formally, at a given day $t$ and for every bid in $\bi \in X$, we replace $n_{j}^{\mathsf{sat}}(\bi)$ and  $e_{j}^{\mathsf{sat}}(\bi)$ in the optimization problem defined in Equation~\eqref{formulation:objectivefunction}--\eqref{formulation:boxconstraints2} with:
\begin{align*}
	&\theta^{(n)}_{j,t-1}(\bi) \sim \mathcal{N}(\hat{\mu}_{j,t-1}(\bi),\hat{\sigma}^2_{j,t-1}(\bi)),\\
	&\theta^{(e)}_{j,t-1}(\bi) \sim \mathcal{N}(\hat{\eta}_{j,t-1}(\bi), \hat{s}^2_{j,t-1}(\bi)),
\end{align*}
respectively.
Similarly, for the value per click, we draw a sample $\theta^{(v)}_{j,t-1}(\bi)$ as follows:
\begin{equation*}
	\theta^{(v)}_{j,t-1} \sim \mathcal{N}(\hat{\nu}_{j,t-1}, \hat{\psi}^2_{j,t-1}).
\end{equation*}
Finally, given the values for $n_{j}^{\mathsf{sat}}(\bi)$ and $e_{j}^{\mathsf{sat}}(\bi)$ generated by one of the two aforementioned methods, we compute $n_j(\bi, \bu)$ as prescribed by Equation~\eqref{eq:nmax} for each bid $\bi \in X$ and for each daily budget $\bu \in Y$, and use them in the following optimization procedure.

\subsection{\textsf{Optimize} Subroutine}
\label{sec:optimmizesubroutine}

Finally, for every campaign $C_j$, we need to choose the best bid/daily budget pair to set at day $t$.
We resort to a modified version of the algorithm in~\cite{kellerer2004multiple} used for the solution of the knapsack problem.
Let us define the set of the feasible bid and daily budgets for the round $t$ and the campaign $C_j$ as
$X_{j,t} := X \cap [\underline{\bi}_{j,t}, \overline{\bi}_{j,t}]$ and $Y_{j,t} := Y \cap [\underline{\bu}_{j,t}, \overline{\bu}_{j,t}]$, respectively.
At first, for every value of daily budget $\bu \in Y_{j,t}$, we define $z_j(\bu) \in X_{j,t}$ as the bid maximizing the number of clicks, formally:
\begin{equation*}
	z_j(\bu) := \arg \max_{\bi \in X_{j,t}} n_j(\bi, \bu).
\end{equation*}
The value $z_j(\bu)$ is easily found by enumeration. 
Then, for each value of daily budget $\bu \in Y$, we define $w_j(\bu)$ as the value we expect to receive by setting the daily budget of campaign $C_j$ equal to $\bu$ and the bid equal to $z_j(\bu)$, formally:
\begin{equation*}
	w_j(\bu) := \left\{ \begin{array}{ll}
		v_j \, n_j(z_j(\bu), \bu) & \underline{\bu}_{j,t} \leq \bu \leq \overline{\bu}_{j,t} \\
		0 & \bu < \underline{\bu}_{j,t} \vee \bu > \overline{\bu}_{j,t}\\
		\end{array} \right..
\end{equation*}
This allows one to remove the dependency of the optimization problem defined in Equations~\eqref{formulation:objectivefunction}--\eqref{formulation:boxconstraints2} from $\bi$, letting variables $\bu$ the only variables to deal with.

Finally, the optimization problem is solved in a dynamic programming fashion.
We use a matrix $M(j, \bu)$ with $j \in \{1, \ldots, N \}$ and $\bu \in Y$.
We fill iteratively the matrix as follows.
Each row is initialized as $M(j, \bu) = 0$ for every $j$ and $\bu \in Y$.
For $j = 1$, we set $M(1, \bu) = w_1(\bu)$ for every $\bu \in Y$, corresponding to the best budget assignment for every value of $\bu$ if the campaign $C_j$ were the only campaign in the problem.
For $j > 1$, we set for every $\bu \in Y$:
\begin{equation*}
	M(j, \bu) = \max_{\bu' \in Y, \bu' \leq \bu} \Big\{ M(j-1, \bu') + w_j(\bu - \bu') \Big\}.
\end{equation*}
That is, the value in each cell $M(j, \bu)$ is found by scanning all the elements $M(j-1, \bu')$ for $\bu'\leq \bu$, taking the corresponding value, adding the value given by assigning a budget of $\bu - \bu'$ to campaign $C_j$ and, finally, taking the maximum among all these combinations.
At the end of the iterative process, the optimal value of the optimization problem can be found in the cell corresponding to $\max_{\bu \in Y} M(N, \bu)$.
To find the optimal assignment of daily budget, it is sufficient to store the partial best assignments of budget in the cells of the matrix.

The complexity of the aforementioned algorithm is $O(N H^2)$, \emph{i.e.}, it is linear in the number of campaigns $N$ and quadratic in the number of different values of the budget $H := |Y|$, where $|\cdot|$ is the cardinality of a set.
%
%
When $H$ is huge, the above algorithm may require a long time.
In that case, it is sufficient to reduce $H$ by rounding the values of the budget as in the FPTAS of the knapsack problem.
%



\section{Regret Analysis} \label{sec:analysis}

We provide a theoretical finite-time analysis of the regret $\mathcal{R}_T(\mathfrak{U})$ of the algorithms proposed in the previous section.
The derivation of the guarantees of our algorithm exploits the results presented in~\cite{accabi2018cmab}.

Initially, we define the \emph{Maximum Information Gain}, which we use to bound the regret of the \alg{} algorithm.
Let us start defining the Information Gain of a set of samples drawn from a GP according to~\cite{srinivas2010gaussian} as follows:
\begin{restatable}[Information Gain]{defi}{informationgain} \label{thm:informationgain}
	Given a realization of a GP $f(\cdot)$ and a vector of noisy observations $\mathbf{y}(\mathbf{x}) = (y(x_1), \ldots, y(x_t))^\top$ over the input points $\mathbf{x} = (x_1, \ldots, x_t)^\top$ for the function $f(\cdot)$, the Information Gain of the set of samples $(\mathbf{x}, \mathbf{y}(\mathbf{x}))$ is defined as:
	\begin{equation*}
		IG(\mathbf{y}(\mathbf{x}) \,|\, f) := \frac{1}{2} \log \left| I + \frac{\Phi}{\lambda} \right|,
	\end{equation*}
	where $I$ is the identity matrix of order $t$, $\lambda$ is the noise variance of the realizations and $[\Phi]_{ij} := k(x_i, x_j)$ is the Gram matrix of the vector computed on the inputs $\mathbf{x}$.
\end{restatable}
Using the previous definition, we define the Maximum Information Gain, as follows.
\begin{restatable}[Maximum Information Gain]{defi}{maxig} \label{thm:maxig}
	Given a realization of a GP $f(\cdot)$, the Maximum Information Gain of a generic set of $t$ noisy observations $ \mathbf{y}(\mathbf{x})$ from the function $f(\cdot)$ is defined as:
	\begin{equation*}
		\gamma_t(f) := \max_{\mathbf{x} \in X^t} IG(\mathbf{y}(\mathbf{x}) \,|\, f),
	\end{equation*}
	where $X$ is the input space.
\end{restatable}

For the sake of presentation, we report our regret analysis separately for the case in which the model of $n_j(\cdot,\cdot)$ is unfactorized (Section~\ref{sec::regret2D}) and the case in which it is factorized (Section~\ref{sec::regretFactored}).

\subsection{Unfactorized Model}
\label{sec::regret2D}

We show that the worst-case pseudo-regret of the \alg{} algorithm when using the unfactorized model is upper bounded as follows.
\begin{restatable}[]{thm}{regretducb} \label{thm:regretducb}
	Let us consider an ABBA problem over $T$ rounds where the functions $n_j(\bi, \bu)$ is the realization of a GP.
	Using the \alg{}-\textsf{U-UCB} algorithm with the following upper bounds for the number of clicks and of value per click:
	\begin{align*}
		&\hat{u}^{(n)}_{j,t-1}(\bi, \bu) := \hat{\mu}_{j,t-1}(\bi, \bu) + \sqrt{b_{t}} \ \hat{\sigma}_{j,t-1}(\bi, \bu),\\
		& \hat{u}^{(v)}_{j,t-1} := \hat{\nu}_{j, t-1} + \sqrt{b'_{t}} \ \hat{\psi}^2_{j,t-1},
	\end{align*}
	respectively, with $b_{t} := 2 \log{ \left( \frac{\pi^2 N M t^2}{3 \delta} \right)}$ and $b'_{t} := 2 \log{ \left( \frac{\pi^2 N t^2}{3 \delta} \right)}$.
	For every $\delta \in (0,1)$, the following holds with probability at least $1-\delta$: 
	\begin{equation*}
		\mathcal{R}_T(\mathfrak{U}) \leq \sqrt{8 T N b_T \left[ \frac{v_{\max}^2}{\log \left( 1 + \frac{1}{\lambda} \right)} \sum_{j=1}^N \gamma_T(n_j) + \xi (n_{\max} + 2 \sqrt{b'_t} \sigma)^2 \sum_{j=1}^N \log \left( \frac{\xi}{\psi_j^2} + T \right) \right]},\\
	\end{equation*}
	where, $\lambda$ and $\xi$ are variances of the measurement noise of the click functions $n_j(\cdot)$ and of the value per click $v_j$, respectively, $v_{\max} := \max_{j \in \{1, \ldots, N \}} v_j$ is the maximum expected value per click, $n_{\max} := \max_{\bi \in X, \bu \in Y, j \in \{1, \ldots, N\}} n_j(\bi, \bu)$ is the maximum expected number of click we might obtain on average over all the campaigns $C_j$, and $\sigma^2 := k(\bm{a}, \bm{a}) \geq \hat{\sigma}^2_{j,t}(\bm{a})$ for each $j$, $t$ and $\bm{a}$.
	Equivalently, with probability at least $1 - \delta$, it holds:
	\begin{equation*}
		\mathcal{R}_T(\mathfrak{U}) = \tilde{O} \left( \sqrt{T N \sum_{j=1}^N \gamma_T(n_j)} \right),
	\end{equation*}
	where the notation $\tilde{O}(\cdot)$ disregards the logarithmic factors.
\end{restatable}

\begin{restatable}[]{thm}{regretdts} \label{thm:regretdts}
	Let us consider an ABBA problem over $T$ rounds where the functions $n_j(\bi, \bu)$ is the realization of a GP.
	Using the \alg{}-\textsf{U-TS} algorithm, for every $\delta \in (0,1)$, the following holds with probability at least $1-\delta$:
	\begin{multline*}
		\mathcal{R}_T(\mathfrak{U}) \leq \Bigg\{ 8 T N \left[ \frac{v_{\max}^2}{\log \left( 1 + \frac{1}{\lambda} \right)} b_T \sum_{j=1}^N \gamma_T(n_j) \right.\\
		\left. + \xi b'_T (n_{\max} + \sqrt{b_T} \sigma)^2 \sum_{j=1}^N \log \left( \frac{\xi}{\psi_j^2} + T \right) \right] \Bigg\}^{1/2},
	\end{multline*}
	where $b_{t} := 8 \log \left( \frac{2 N M t^2}{3 \delta} \right)$, $b'_{t} := 8 \log \left( \frac{2 N t^2}{3 \delta} \right)$, $\lambda$ and $\xi$ are variances of the measurement noise of the click functions $n_j(\cdot)$ and of the value per click $v_j$, respectively, $v_{\max} := \max_{j \in \{1, \ldots, N \}} v_j$ is the maximum expected value per click, $n_{\max} := \max_{\bi \in X, \bu \in Y, j \in \{1, \ldots, N\}} n_j(\bi, \bu)$ is the maximum expected number of click we might obtain on average over all the campaigns $C_j$, and $\sigma^2 := k(\bm{a}, \bm{a}) \geq \hat{\sigma}^2_{j,t}(\bm{a})$ for each $j$, $t$ and $\bm{a}$.	
		
	Equivalently, with probability at least $1 - \delta$, it holds:
	\begin{equation*}
		\mathcal{R}_T(\mathfrak{U}) =  \tilde{O} \left( \sqrt{T N \sum_{j=1}^N \gamma_T(n_j)} \right).
		\end{equation*}
\end{restatable}

\subsection{Factorized Model}
\label{sec::regretFactored}
We show that the worst-case pseudo-regret of the \alg{} algorithm  when using the factorized model is upper bounded as follows.

\begin{restatable}[]{thm}{regretfucb} \label{thm:regretfucb}
	Let us consider an ABBA problem over $T$ rounds where the functions $n_{j}^{\mathsf{sat}}(\bi)$ and $e^{\mathsf{sat}}_{j}(\bi)$ are the realization of GPs.
	Using the \alg{}-\textsf{F-UCB} algorithm with the following upper bounds for the number of clicks, the number of clicks per unit of budget, and the value per click, respectively:
	\begin{align*}
		&u^{(n)}_{j,t-1}(\bi) := \hat{\mu}_{j,t-1}(\bi) + \sqrt{b_{t}} \hat{\sigma}_{j,t-1}(\bi),\\
		&u^{(e)}_{j,t-1}(\bi) := \hat{\eta}_{j,t-1}(\bi) + \sqrt{b_{t}} \hat{s}_{j,t-1}(\bi),\\
		&u^{(v)}_{j,t-1} := \hat{\nu_{j,t-1}} + \sqrt{b'_{t}} \hat{\psi}_{j,t-1},
\end{align*}
	with $b_{t} = 2 \log \left( \frac{\pi^2 N M t^2}{2 \delta} \right)$ and $b'_{t} := 2 \log \left( \frac{\pi^2 N t^2}{2 \delta} \right)$.
	For every $\delta \in (0, 1)$, the following holds with probability at least $1 - \delta$:
	\begin{multline*}
		\mathcal{R}_T(\mathfrak{U}) \leq \Bigg\{ T N \left[ \bar{c}_1 b_T \sum_{j=1}^{N} \gamma_T(n_j) + \bar{c}_2 b_T \sum_{j=1}^{N} \gamma_T(e_j) \right.\\
		\left. + \bar{c}_3 b'_T \left(2 s \bu_{\max} \sqrt{b_T} + 2 \sigma \sqrt{b_T} + n_{\max}^{\mathsf{sat}} \right)^2 \sum_{j=1}^{N} \log \left( \frac{\xi}{\psi^2_j} + T \right) \right]\Bigg\}^{1/2},
	\end{multline*}
	where $\bar{c}_1 := \frac{12 v^2_{\max}}{\log \left( 1 + \frac{1}{\lambda} \right)}$, $\bar{c}_2 := \frac{12 v^2_{\max} \bu^2_{\max}}{\log \left( 1 + \frac{1}{\lambda'} \right)}$, and $\bar{c}_3 := 12 \xi$, $\xi$, $\lambda$ and $\lambda'$ are the variance of the value per click, measurement noise on the maximum number of clicks and number of clicks per unit of daily budget, respectively, $v_{\max} := \max_{j \in \{1, \ldots, N \}} v_j$ is the maximum expected value per click, $n_{\max} := \max_{\bi \in X, \bu \in Y, j \in \{1, \ldots, N\}} n_j(\bi, \bu)$ is the maximum expected number of click we might obtain on average over all the campaigns $C_j$, $\bu_{\max} := \max_{\bu \in Y} \bu$ is the maximum budget one can allocate on a campaign, and $\sigma^2 := k(\bi, \bi) \geq \hat{\sigma}^2_{j,t}(\bi)$, $s^2 := k'(\bi, \bi) \geq \hat{s}^2_{j,t}(\bi)$ for each $j$, $t$ and $\bi$.
	
	Equivalently, with probability at least $1-\delta$, it holds:
	\begin{equation*}
	\mathcal{R}_T(\mathfrak{U}) = \tilde{O} \left( \sqrt{T N \sum_{j=1}^N [\gamma_T(n_j) +  \gamma_T(e_j)]} \right).
	\end{equation*}
\end{restatable}

\begin{restatable}{thm}{regretfts} \label{thm:regretfts}
	Let us consider an ABBA problem over $T$ rounds where the functions $n_{j}^{\mathsf{sat}}(\bi)$ and $e^{\mathsf{sat}}_{j}(\bi)$ are the realization of GPs.
	Using the \alg{}-\textsf{F-TS} algorithm, for every $\delta \in (0,1)$, the following holds with probability at least $1-\delta$:
	\begin{multline*}
		\mathcal{R}_T(\mathfrak{U}) \leq \Bigg\{ T N \left[ \bar{c}_1 b_T \sum_{j=1}^{N} \gamma_T(n_j) + \bar{c}_2 b_T \sum_{j=1}^{N} \gamma_T(e_j) \right.\\
		\left. + \bar{c}_3 b'_T \left(2 s \bu_{\max} \sqrt{b_T} + 2 \sigma \sqrt{b_T} + n_{\max}^{\mathsf{sat}} \right)^2 \sum_{j=1}^{N} \log \left( \frac{\xi}{\psi^2_j} + T \right) \right]\Bigg\}^{1/2},
	\end{multline*}
	where $b_{t} = 2 \log \left( \frac{\pi^2 N M t^2}{2 \delta} \right)$, $b'_{t} := 2 \log \left( \frac{\pi^2 N t^2}{2 \delta} \right)$, $\bar{c}_1 := \frac{48 v^2_{\max}}{\log \left( 1 + \frac{1}{\lambda} \right)}$, $\bar{c}_2 := \frac{48 v^2_{\max} \bu^2_{\max}}{\log \left( 1 + \frac{1}{\lambda'} \right)}$, and $\bar{c}_3 := 12 \xi$, $\xi$, $\lambda$ and $\lambda'$ are the variance of the value per click, measurement noise on the maximum number of clicks and number of clicks per unit of daily budget, respectively, $v_{\max} := \max_{j \in \{1, \ldots, N \}} v_j$ is the maximum expected value per click, $n_{\max} := \max_{\bi \in X, \bu \in Y, j \in \{1, \ldots, N\}} n_j(\bi, \bu)$ is the maximum expected number of click we might obtain on average over all the campaigns $C_j$, $\bu_{\max} := \max_{\bu \in Y} \bu$ is the maximum budget one can allocate on a campaign, and $\sigma^2 := k(\bi, \bi) \geq \hat{\sigma}^2_{j,t}(\bi)$, $s^2 := k'(\bi, \bi) \geq \hat{s}^2_{j,t}(\bi)$ for each $j$, $t$ and $\bi$.
	
	Equivalently, with probability at least $1-\delta$, it holds:
	\begin{equation*}
		\mathcal{R}_T(\mathfrak{U}) = \tilde{O} \left( \sqrt{T N \sum_{j=1}^N [\gamma_T(n_j) +  \gamma_T(e_j)]} \right).
	\end{equation*}
\end{restatable}

The upper bounds provided by Theorems~\ref{thm:regretducb}--\ref{thm:regretfts} are expressed in terms of the maximum information gain $\gamma_T(\cdot)$ one might obtain selecting $T$ samples from the GPs defined in Section~\ref{sec:updatesubroutine}.
The problem of bounding $\gamma_T(f)$ for a generic GP $f$ has been already addressed by~\cite{srinivas2010gaussian}, where the authors present the bounds for the squared exponential kernel $\gamma_T(f) = O((\log{T})^{d+1})$, where $d$ is the dimension of the input space of the GP ($d = 2$ for \alg{}-\textsf{U}, and $d = 1$ for \alg{}-\textsf{F}).
Notice that, thanks to the previous result our \alg{} algorithm suffers from a sublinear pseudo-regret since the terms $\gamma_{T}(n_j)$ and $\gamma_T(e_j)$ are bounded by $O( (\log{T})^{d+1})$, and the bound in Theorems~\ref{thm:regretducb}--\ref{thm:regretfts} is then $O(N \sqrt{T (\log{T})^{d+1})})$.
%
%

\section{Experimental Evaluation}
\label{sec:experiments}

This section is structured as follows.
In Section~\ref{sec:experimentalevaluation}, we experimentally evaluate the convergence to the optimal solution and the empirical regrets of our algorithms in synthetic settings.
In Section~\ref{sec:experimentalevaluationreal}, we present the results of the adoption of our algorithms in a real-world setting.


\subsection{Evaluation in the Synthetic Setting}\label{sec:experimentalevaluation}

We evaluate our algorithms in synthetic settings generated as follows.
In every setting, there is a single advertiser optimizing a set of $N$ campaigns by our algorithms, and, for every campaign $C_j$, there are other $\delta_j-1$ advertisers whose behavior is, instead, stochastic. 
At day $t$, every campaign $C_j$ can be involved in a set $AU_{j,t}$ of auctions, whose number $|AU_{j,t}|$ is drawn from a Gaussian probability distribution $\mathcal{N}(\mu^{(s)}_j, (\sigma^{(s)}_j)^2)$ with mean $\mu^{(s)}_j$ and standard deviation $\sigma^{(s)}_j$ and subsequently rounded to the nearest integer.
We denote, for campaign $C_j$, the click and conversion probabilities of the advertiser using our algorithms with $p^{(cl)}_j$ and $p^{(co)}_j$, respectively.
These probabilities are the same for all the auctions in which $C_j$ is involved. 
We assign a tuple of parameters $\mu^{(b)}, (\sigma^{(b)})^2$ to every other advertiser before the beginning of the experiment, and, at every auction, the bids $b_h$ are drawn from a Gaussian distribution $\mathcal{N}(\mu^{(b)}, (\sigma^{(b)})^2)$, being $\mu^{(b)}$ and $\sigma^{(b)}$ the mean and standard deviation parameters of the bid distribution, respectively.
Similarly, the click probabilities $\rho_h$ are uniformly sampled in the interval $[0, 1]$ at every auction.

The auction mechanism we use is the Vickrey-Clarke-Groves~\cite{mas1995microeconomic} and the number of available slots is  $\gamma_j$, with $\gamma_j \leq \delta_j$.
%
%
%
%
Once the optimal allocation is found, we simulate a user who may or may not click the ad, and generate a click and/or a conversion according to probabilities $p^{(cl)}_{j}$ and $p^{(co)}_{j}$, respectively.
After the click, the daily budget of the advertiser using our algorithms is reduced as prescribed by the Vickrey-Clarke-Groves mechanism.

\begin{table}[t!]
\caption{Parameters of the synthetic settings.} \label{tab:para}
\scriptsize
\centering
\begin{tabu}{|l|[2pt]l|l|l|l|}
\hline
& $C_1$  & $C_2$ & $C_3$ & $C_4$ \\ \tabucline[2pt]{-}
$\mu_j$       & 1000  & 1500 & 1500  & 1250  \\  \hline
$\sigma_j$    & 50    & 50   & 50   & 50 \\ \hline
$\gamma_j$  & 5     & 5    & 5    & 5  \\  \hline
$\delta_j$    & 7     & 7    & 7    & 7 \\  \hline
$\mu^{(b)}$    & 0.5  & 0.33 & 0.4 & 0.39  \\  \hline
$\sigma^{(b)}$ & 0.1  & 0.07 & 0.1  & 0.51   \\  \hline
$p^{(obs)}(1)$              & 0.9   & 0.9  & 0.9  & 0.9  \\  \hline
$p^{(obs)}(2)$              & 0.7   & 0.8  & 0.7  & 0.8   \\  \hline
$p^{(obs)}(3)$              & 0.6   & 0.7  & 0.6  & 0.6  \\  \hline
$p^{(obs)}(4)$              & 0.4   & 0.6  & 0.4  & 0.5   \\  \hline
$p^{(obs)}(5)$              & 0.2   & 0.5  & 0.3  & 0.3   \\  \hline
$p^{(cl)}_j$                     & 0.5   & 0.3  & 0.4  & 0.4  \\  \hline
$p^{(co)}_j$                     & 0.05 & 0.05 & 0.04  & 0.05  \\  \hline
\end{tabu}
\end{table}

\subsubsection{Experiment $\#1$}

This experiment aims at showing that the algorithms, which do not sufficiently explore the space of the arms, may not converge to the (clairvoyant) optimal solution.
We use a setting with $N = 4$ different campaigns.
The parameters describing the setting are provided in Table~\ref{tab:para}.
Furthermore, in Figure~\ref{fig:n_functions}, we report, for each campaign $C_j$, the best instantaneous revenue $v_j \max_x n_j(x, y)$ as the daily budget allocated to the single campaign varies (this is done maximizing the performance over the feasible bid values $\bi \in X$).
The peculiarity of this setting is the similarity of the performance of campaigns $C_1$ and $C_4$.
Indeed, this similarity makes the identification of the optimal solution hard.
We set the following limits for every $t \leq T$ where $T = 200$ days and for every campaign $C_j$: cumulative budget $\overline{\bu}_t = 500$, minimum and maximum bid values $\underline{\bi}_{j,t} = 0$, and $\overline{\bi}_{j,t} = 2$, respectively, minimum and maximum daily budget values $\underline{\bu}_{j,t} = 0$, $\overline{\bu}_{j,t} = 500$, respectively.
Furthermore, we use an evenly spaced discretization of $|X| = 10$ bids and $|Y| = 10$ budgets over the aforementioned intervals.
We assume a uniform distribution of the clicks and conversions over the day (see Section~\ref{sec:method}).

We compare the experimental results of our algorithms (\textsf{\alg-U-UCB}, \textsf{\alg-U-TS} \textsf{\alg-F-UCB}, and \textsf{\alg-F-TS}) to identify the modeling and/or exploration strategies providing the best performance.~\footnote{Notice that a straightforward extension of the algorithm proposed in~\cite{chen2013combinatorial}, \emph{i.e.}, designing a version accounting for Gaussian distribution, would require $100$ days to have a single sample for each different bid/daily budget pair.
Indeed, it would purely explore the space of arms without any form of exploitation for $t \leq 100$.
Therefore, we omit the comparison that would not provide any meaningful insight to the problem.
}
Furthermore, we introduce a baseline represented by the algorithm \textsf{\alg-F-MEAN}, which is a ``pure exploration'' version of \textsf{\alg-F} such that, at every round $t$, the posterior expected value of the number of clicks for each bid/daily budget pair is given in input to the optimization procedure.
In the GPs used by all the algorithms, we adopt a squared exponential kernel, whose hyper-parameters are chosen as prescribed by the GP literature, see~\cite{rasmussen2006gaussian} for details, and we started from an uninformative zero-mean prior.

In this experiment, in addition to the cumulative pseudo-regret $R_t(\mathfrak{U})$, we also evaluate the expected value of the revenue $r_{\bm{\mu}}(S_t)$.
Obviously, in the case of the cumulative pseudo-regret, the performance improves as the cumulative pseudo-regret reduces, and, conversely, in the case of the revenue, the performance improves as it increases.
The experimental results are averaged over $100$ independent executions of the algorithms.

\begin{figure}[th!]
\centering
\includegraphics{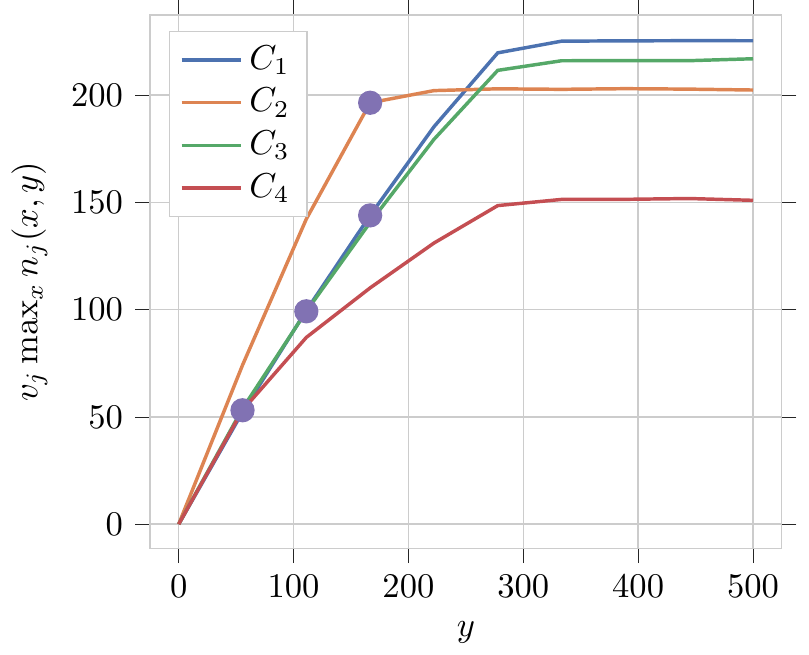}
\caption{Expected value of $v_j \max_x n_j(x, y)$ used in Experiment $\#1$ for each campaign and for each value of the daily budget.
Violet dots corresponds to the expected value of the number of conversions associated to the optimal daily budget allocation.}
\label{fig:n_functions}
\end{figure}

\begin{figure*}[th!]
\subfloat[]{\includegraphics{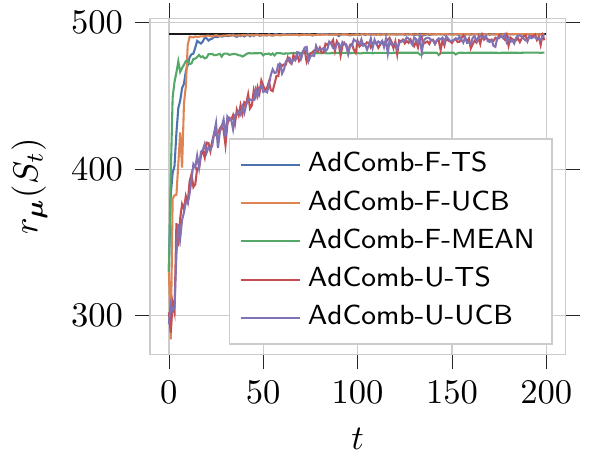}\label{fig:reward}}
\subfloat[]{\includegraphics{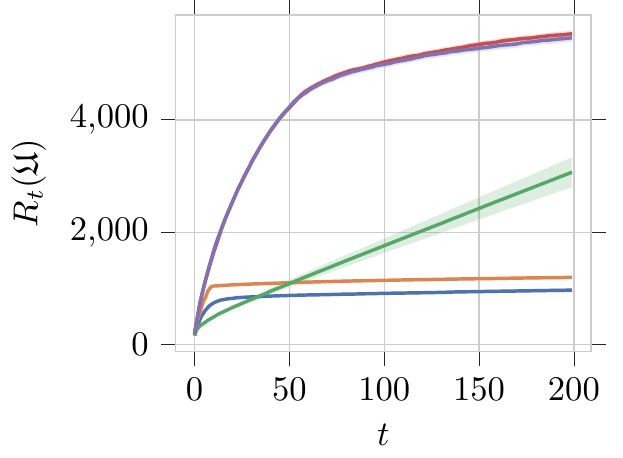}\label{fig:regret}}
\caption{Results for Experiment $\#1$: revenue (a), and cumulative pseudo-regret (b). The black horizontal line in (a) is the optimal reward of the clairvoyant algorithm $r_{\bm{\mu}}^*$. The shaded regions in (b) represent the 95\% confidence intervals of the mean.}
\end{figure*}

In Figure~\ref{fig:reward}, we report the average instantaneous reward $r_{\bm{\mu}}(S_t)$ of our algorithms, while, in Figure~\ref{fig:regret}, we report their average cumulative pseudo-regret $R_t(\mathfrak{U})$.
The reward provided by all the algorithms but \textsf{\alg-F-MEAN} converges to the optimal reward provided by a clairvoyant algorithm and presents a slightly varying reward even at the end of the time horizon due the variance of the GP used to choose the daily budget allocation over time.
This variance is larger at the beginning of the process, thus incentivising exploration, and it decreases as the number of observations increases, allowing the algorithms to reach the (clairvoyant) optimal reward asymptotically.
Although all our algorithms converge to the (clairvoyant) optimal solution, the \textsf{\alg-F-TS} algorithm provides the smallest cumulative pseudo-regret for every $t \geq 30$.
\textsf{\alg-F-UCB} has performance slightly worse than that of \textsf{\alg-F-TS}.
The \textsf{\alg-F-MEAN} algorithm provides the best performance for $t \leq 30$, but it is not capable to achieve the (clairvoyant) optimal solution.
As a result, for larger values of $t$, the performance of \textsf{\alg-F-MEAN} decreases achieving, at $t = 200$, a regret significantly larger than that one provided by \textsf{\alg-F-TS}.
This is because \textsf{\alg-F-MEAN} does not explore the arms space properly and, as a consequence, in some of the $100$ independent runs, it gets stuck in a suboptimal solution of the optimization problem.
Conversely, \textsf{\alg-F-TS} and \textsf{\alg-F-UCB}, thanks to their exploration incentives, converge to the optimal solution asymptotically in all the runs.
We have a similar behavior of the algorithms is a situation in which the performance of the algorithms are rather different and the observations are very noisy. 
Finally, we observe that \textsf{\alg-U-UCB} and \textsf{\alg-U-TS} suffer from a much larger regret than that one of their factorized counterparts---more than $100\%$ at $t = 200$---and this is mainly accumulated over the first half of the time horizon.

In real-world settings, it may be usual dealing with scenarios in which multiple campaigns have similar performance, or they have different performance and the observations are very noisy. 
In those situations, \textsf{\alg-F-MEAN} might get stuck in a suboptimal solution, thus providing a small expected revenue w.r.t.~a clairvoyant algorithm.
Conversely, both the unfactorized and the factorized versions of our algorithms might still be a viable solution since they have proven to converge to the optimum asymptotically.
For this reason, we do not recommend the adoption of the \textsf{\alg-F-MEAN} algorithm in practice and we omit its evaluation in the following experimental activities.

\FloatBarrier

\subsubsection{Experiment $\#2$}

%
This experiment aims at evaluating how the size of the discretization of the bid space (in terms of $|X|$) and daily budget space (in terms of  $|Y|$) used in our algorithms affect their performance.
Indeed, if, on the one hand, an increase in the number of the available bid/daily budget pairs corresponds to an increase of the expected revenue of the clairvoyant solution, on the other hand, a larger arms space results in larger exploration costs. %
We investigate how the impact on the exploration cost is mitigated by the correlation between arms that allows one to gain information over all the arms space once one arm is pulled.

In this experiment, we adopt the same setting used in Experiment $\#1$ (see Table~\ref{tab:para}) with different granularities of discretization in the bid space $X$ or daily budget space $Y$.
In particular, we study two settings over a time horizon of $T = 50$ rounds.
In the first setting, the number of values of the bid space is $|X| \in \{5, 10, 20, 40, 80 \}$, while the number of daily budget values is $|Y| = 10$.
In the second setting, the number of values of the daily budget space is $|Y| \in \{5, 10, 20, 40, 80 \}$, while the number of bid values is $|X| = 10$. 
These discretizations are such that every space with a larger number of values includes the space with a smaller number of values.
For instance, a set with $40$ values strictly includes the one with $20$ values.

We compare \textsf{\alg{}-{U-TS}} and \textsf{\alg{}-{F-TS}} in terms of both cumulative expected revenue (over time) $P_T(\mathfrak{U}) := \sum^{T}_{t = 1} r_{\bm{\mu}}(S_t)$ and the following performance index:
\begin{equation*}
	V(X, Y) = \frac{P_T(\mathfrak{U})}{T \,r_{\bm{\mu}}^*},
\end{equation*}
where both the algorithms selecting $S_t$ and the clairvoyant algorithm selecting $S^*$ (corresponding to $r_{\bm{\mu}}^*$) are run on the space $X \times Y$.~\footnote{
The results for \textsf{\alg{}-{U-UCB}} and \textsf{\alg{}-{F-UCB}} are omitted since they are in line with the ones we present and do not provide any further insight.}
Basically, $V(X,Y) \in [0, 1]$ is a ratio returning, given a space of arms $X \times Y$, the efficiency of a learning algorithm with respect to the optimal solution achievable with that arm space.
The normalization with respect to the optimal solution achievable with a given arm space mitigates the fact that, enlarging the arms space, the optimal revenue may increase.

\begin{figure*}[th!]
\centering
\subfloat[]{\includegraphics{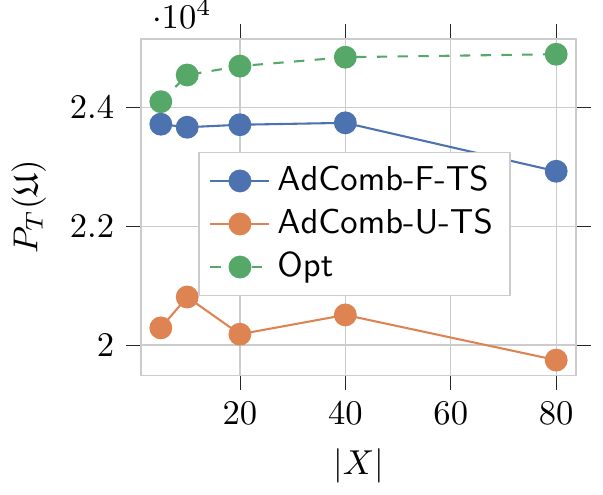}\label{fig:opt_bid_discr}}
\subfloat[]{\includegraphics{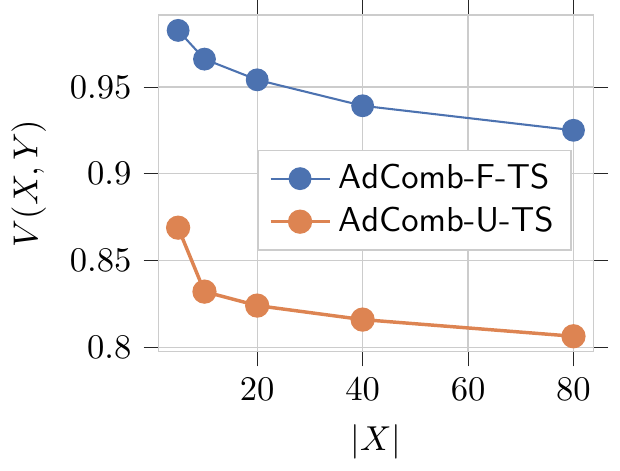}\label{fig:comp_bid_discr}}

\caption{Results for Experiment $\#2$: performance of the \textsf{\alg{}-F-TS} and \textsf{\alg{}-U-TS} algorithms as $|X|$ varies: (a) cumulative expected revenue, (b) $V(X, Y)$ index.}
\end{figure*}

\begin{figure*}[th!]
\centering
\subfloat[]{\includegraphics{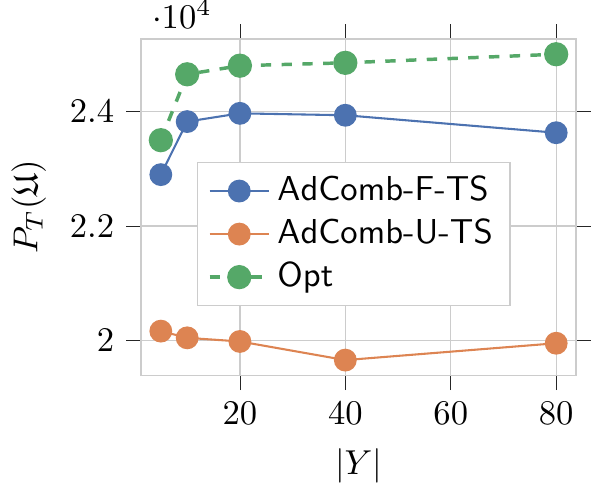}\label{fig:opt_bud_discr}}
\subfloat[]{\includegraphics{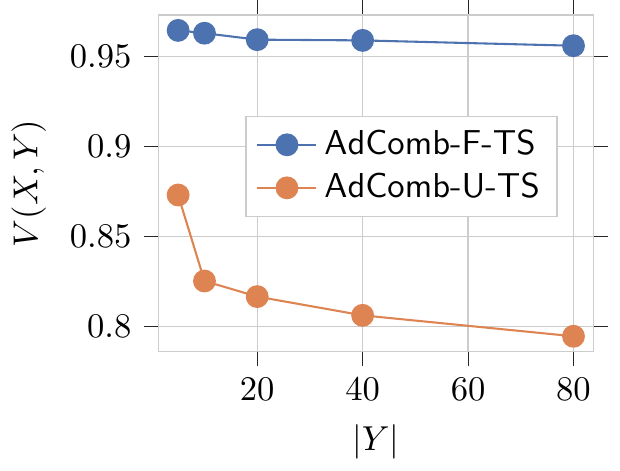}\label{fig:comp_bud_discr}}

\caption{Results for Experiment $\#2$: performance of the \textsf{\alg{}-F-TS} and \textsf{\alg{}-U-TS} algorithms as $|Y|$ varies: (a) cumulative expected revenue, (b) $V(X, Y)$ index.}
\end{figure*}

In Figure~\ref{fig:opt_bid_discr}, we show the value of $P_T(\mathfrak{U})$ for our algorithms and of the optimal solution (denoted with \textsf{Opt}) as the bid space granularity $|X|$ varies.
It is worth noting that the increase in the clairvoyant optimal reward for $|X| \geq 10$ is negligible.
The cumulative revenue provided by the \textsf{\alg{}-{F-TS}} algorithm is decreasing for $|X| > 40$. However, the loss w.r.t.~the revenue gained when $|X| = 80$ is about $3\%$.
The \textsf{\alg{}-{U-TS}} algorithm has a similar behavior. 
This result shows that the cost due to the exploration of a larger space of arms is larger than the increase of the optimal achievable reward, and it suggests, in practice, the adoption of a discretization of the bid space with about $|X| = 40$ evenly-spaced values.
In Figure~\ref{fig:comp_bid_discr}, we show the values of $V(X, Y)$ achieved by the \textsf{\alg{}-{U-TS}} and \textsf{\alg{}-{F-TS}} algorithms.
In this case, for both algorithms, the value of $V(X, Y)$ decreases as $|X|$ increases.
This is because the algorithms pays a larger exploration cost.
However, the empirical increase in inefficiency is only logarithmic in $|X|$.
This result shows that the performance of the algorithms is robust to an increase of the number of possible bids.

In Figure~\ref{fig:opt_bud_discr}, we show the value $P_T(\mathfrak{U})$ of our algorithms and of the optimal solution as $|Y|$ varies.
The results are similar to those obtained above when $|X|$ varies.
The only peculiarity, in this case, concerns the performance of the \textsf{\alg{}-{F-TS}} algorithm.
Despite the theoretical analysis provides a logarithmic dependency of the regret from $|Y|$, the empirical results seem to suggest that the values of $V(X, Y)$ are approximately constant as $|Y|$ varies.
This empirical result does not hold for the \textsf{\alg{}-{U-TS}} algorithm, whose performance are significantly affected by the number of daily budget intervals.
Intuitively, this is because, in the factorized model, the daily budget is not an input to the GPs we use, thus, it does not affect the exploration of the algorithm.
Conversely, in \textsf{\alg{}-{U-TS}}, the daily budget values constitute a part of the input to the GPs, and, therefore, an increase of the number of the daily budget values results in a larger regret since the algorithm needs to explore a wider space of arms.

In conclusion, the use of a more fine-grained bid/daily budget space to explore provides less revenue overall, but with a mild decrease in terms of performance, allowing, in practical cases, to use a large discretization space.
\FloatBarrier

\subsubsection{Experiment $\#3$}

This experiment aims at evaluating the performance of our algorithms with random realistic settings generated by exploiting \emph{Yahoo!} Webscope $A3$ dataset. 
More specifically, we consider $N = 4$ campaigns whose parameters $\mu_j$, $\sigma_j$, $\gamma_j$, $\delta_j$ are those reported in Table~\ref{tab:para}. 
The values of the parameters $\mu^{(b)}$, $\sigma^{(b)}$, $p^{(obs)}(\delta)$, $p^{(cl)}_j$ are, instead, generated according to distributions estimated from the auctions of the \emph{Yahoo!} Webscope $A3$ dataset. 
We set a constant cumulative daily budget $\overline{\bu}_t = 100$ over a time horizon of $T = 100$ days, with limits $\underline{\bu}_{j,t} = 0$, $\overline{\bu}_{j,t} = 100$, $\underline{\bi}_{j,t} = 0$, and $\overline{\bi}_{j,t} = 1$ for every $t \leq T,C_j$.
Furthermore, we use an evenly spaced discretization of $|X| = 10$ values of bid and $|Y| = 10$ values of daily budget.
We generate $10$ different scenarios, for each of them, we run $100$ independent experiments over the same scenarios and averaged over them.
Given a setting and an algorithm, we denote with $\beta$ the percentage of runs in which the given algorithm has the best performance in terms of cumulative reward in the given setting.

In Table~\ref{tab:rand_performance}, we report, for every algorithm and every setting, the average cumulative regret $R_T$, its standard deviation $\sigma_{R_T}$, and $\beta$. 
In almost all the settings, the best algorithm is \textsf{\alg{}-F-TS}. 
Furthermore, \textsf{\alg{}-F-TS} outperforms the other algorithms in more than the $70\%$ of the runs in all settings. 
However, it is worth to note that in settings $1-5-6$, for $t \leq 25$, \textsf{\alg-{F-UCB}} outperforms the other algorithms in more than $11\%, 73\%, 7\%$ of the cases, respectively. 
This provides evidence that only in some specific scenarios \textsf{\alg-{F-UCB}} provides a viable solution to the ads optimization problem.
Conversely, \textsf{\alg-{U-TS}} and \textsf{\alg-{U-UCB}} achieve lower performance than \textsf{\alg-{F-TS}} and \textsf{\alg-{F-UCB}}  algorithms in all runs.
For this reason, in the real-world setting, we adopt the \textsf{\alg{}-F-TS} algorithm.

\begin{table}

	\centering
\renewcommand{\arraystretch}{1}
	\resizebox{\columnwidth}{!}{\begin{tabular}{l|l||c|c|c||c|c|c||c|c|c||c|c|c|}
\multicolumn{2}{c||}{}& \multicolumn{3}{c||}{\textsf{\alg{}-F-TS}} & \multicolumn{3}{c||}{\textsf{\alg{}-F-UCB}} & \multicolumn{3}{c||}{\textsf{\alg{}-U-TS}} & \multicolumn{3}{c|}{\textsf{\alg{}-U-UCB}}  \\ \cline{3-14}
\multicolumn{2}{c||}{}& $R_T$ & $\sigma_{R_T}$ & $\beta$ & $R_T$ & $\sigma_{R_T}$ & $\beta$ & $R_T$ & $\sigma_{R_T}$ & $\beta$ &$R_T$ & $\sigma_{R_T}$ & $\beta$  \\ \hline\hline
Setting 1 & $t=25$ & 73  & 13 & $89 \%$ & 104 & 21 & $11\%$& 287 & 24 &$0\%$ &264 & 23 &$0\%$  \\ \cline{2-14}
          & $t=50$  & 111 & 22 & $97 \%$ & 169 & 24 & $3\%$ & 377 & 28 &$0\%$ &345 & 30 &$0\%$  \\ \cline{2-14}
          & $t=100$ & 170 & 37 & $99 \%$ & 279 & 29 & $1\%$ & 485 & 35 &$0\%$ &444 & 36 &$0\%$  \\ \hline\hline
Setting 2 & $t=25$  & 66  & 16 & $98 \%$ & 125 & 15 & $2\%$ & 255 & 21 &$0\%$ &263 & 21 &$0\%$  \\ \cline{2-14}
          & $t=50$  & 93  & 21 & $ 98\%$ & 155 & 18 & $2\%$ & 327 & 26 &$0\%$ &340 & 24 &$0\%$  \\ \cline{2-14}
          & $t=100$ & 132 & 31 & $92 \%$ & 194 & 22 & $8\%$ & 415 & 28 &$0\%$ &424 & 30 &$0\%$  \\ \hline\hline
Setting 3 & $t=25$  & 75  & 15 & $100\%$ & 132 & 18 & $0\%$ & 319 & 31 &$0\%$ &316 & 30 &$0\%$  \\ \cline{2-14}
          & $t=50$  & 103 & 21 & $100\%$ & 182 & 20 & $0\%$ & 421 & 37 &$0\%$ &397 & 32 &$0\%$  \\ \cline{2-14}
          & $t=100$ & 145 & 32 & $100\%$ & 261 & 24 & $0\%$ & 533 &  42&$0\%$ &488 & 40 &$0\%$  \\ \hline\hline
Setting 4 & $t=25$  & 67  & 15 & $100\%$ & 130 & 29 & $0\%$ & 334 & 27 &$0\%$ &306 & 26 &$0\%$  \\ \cline{2-14}
          & $t=50$  & 103 & 19 & $100\%$ & 196 & 38 & $0\%$ & 414 & 31 &$0\%$ &395 & 29 &$0\%$  \\ \cline{2-14}
          & $t=100$ & 164 & 29 & $100\%$ & 297 & 55 & $0\%$ & 512 & 38 &$0\%$ &499 & 35 &$0\%$  \\ \hline\hline
Setting 5 & $t=25$  & 112 & 18 & $27\% $ & 99  & 12 & $73\%$& 345 & 30 &$0\%$ &321 & 25 &$0\%$  \\ \cline{2-14}
          & $t=50$  & 157 & 22 & $83\% $ & 180 & 13 & $17\%$& 479 & 39 &$0\%$ &457 & 30 &$0\%$  \\ \cline{2-14}
          & $t=100$ & 222 & 24 & $100\%$ & 331 & 14 & $0\%$ & 648 & 62 &$0\%$ &627 & 36 &$0\%$  \\ \hline\hline
Setting 6 & $t=25$  & 100 & 15 & $ 93\%$ & 99  & 9  & $7\%$ & 272 & 20 &$0\%$ &287 & 24 &$0\%$  \\ \cline{2-14}
          & $t=50$  & 140 & 19 & $100\%$ & 180 & 13 & $0\%$ & 370 & 32 &$0\%$ &391 & 28 &$0\%$  \\ \cline{2-14}
          & $t=100$ & 221 & 32 & $100\%$ & 331 & 20 & $0\%$ & 480 & 37 &$0\%$ &507 & 34 &$0\%$  \\ \hline\hline
Setting 7 & $t=25$  & 100 & 15 & $98\% $ & 142 & 14 & $2\%$ & 336 & 24 &$0\%$ &344 & 28 &$0\%$  \\ \cline{2-14}
		  & $t=50$  & 145 & 19 & $94\%$ & 184 & 13 & $6\%$ & 453 & 34 &$0\%$ &456 & 33 &$0\%$  \\ \cline{2-14}
		  & $t=100$ & 220 & 30 & $83\%$ & 250 & 14 & $17\%$ & 595 & 43 &$0\%$ &587 & 37 &$0\%$ \\ \hline\hline
Setting 8 & $t=25$  & 90  & 14 & $ 98\%$ & 144 & 22 & $2\%$ & 296 & 30 &$0\%$ &278 & 24 &$0\%$ \\ \cline{2-14}
		  & $t=50$  & 128 & 17 & $99\%$ & 220 & 21 & $1\%$ & 386 & 30 &$0\%$ &363 & 28 &$0\%$  \\ \cline{2-14}
		  & $t=100$ & 181 & 20 & $100\%$ & 303 & 27 & $0\%$ & 495 & 35 &$0\%$ &466 & 34 &$0\%$ \\ \hline\hline
Setting 9 & $t=25$  & 89  & 17 & $ 90\%$ & 127 & 21 & $10\%$ & 334 & 23 &$0\%$ &329 & 26 &$0\%$  \\ \cline{2-14}
		  & $t=50$  & 119 & 23 & $91\%$ & 162 & 24 & $9\%$ & 417 & 28 &$0\%$ &421 & 34 &$0\%$  \\ \cline{2-14}
		  & $t=100$ & 165 & 33 & $88\%$ & 224 & 32 & $12\%$ & 510 & 36 &$0\%$ &519 & 40 &$0\%$ \\ \hline\hline
Setting 10 & $t=25$  & 91 & 20 & $ 100\%$ & 197  & 10  & $7\%$ & 331 & 25 &$0\%$ &302 & 29 &$0\%$  \\ \cline{2-14}
		  & $t=50$  & 138 & 23 & $100\%$ & 248 & 16 & $0\%$ & 424 & 31 &$0\%$ &402 & 39 &$0\%$  \\ \cline{2-14}
		  & $t=100$ & 214 & 35 & $100\%$ & 333 & 22 & $0\%$ & 537 & 38 &$0\%$ &518 & 44 & $0\%$ \\ \hline\end{tabular}}\caption{Results for Experiment $\#3$: performance of algorithms \textsf{\alg{}-F-TS}, \textsf{\alg{}-F-UCB}, \textsf{\alg{}-U-TS}, \textsf{\alg{}-U-UCB} with 10 different random settings at round $t \in \{25,50,100\}$.}
		  \label{tab:rand_performance}
\end{table}

\subsection{Evaluation in a Real-world Setting}
\label{sec:experimentalevaluationreal}

We adopted the \textsf{AdComB-F-TS} algorithm to advertise in Italy a set of campaigns for a loan product of a large international company.
The goal was the maximization of the number of the leads. 
Due to reasons of industrial secrecy, we cannot disclose the name of the product and the name of the company. 
The campaigns started July $1^{st}$ $2017$ and were active up to December $31^{st}$ $2019$.
In our discussion, we report the results corresponding to the first $365$ days of the experiment, grouped by weeks as the behavior of the users can be slightly different during the days of a single week and some campaigns were active only some specific days of the week.
During these $365$ days, the set of campaigns changed over time, some being added, others being discarded or changed, due to business needs such as, \emph{e.g.}, the creation of new graphical logos, messages, or new user profiles to target.
In particular, the total number of campaigns used is $29$, while, initially, the campaigns were $8$.
The activation/deactivation of the campaigns in time and their actual costs per week are depicted in Figure~\ref{fig:matrix}.
After $365$ days, the previously active campaigns were completely discarded and a new set of campaigns was used.

The campaigns were optimized by human specialists from week 0 to week $5^{\textnormal{th}}$ and by our algorithm from week $6^{\textnormal{th}}$ on. 
In addition, the cumulative daily budget was changed at weeks $6^{\textnormal{th}}$ and $29^{\textnormal{th}}$ due to business reasons of the company, as follows:
\begin{itemize}
\item $200$ Euros per day from week $0^{\textnormal{th}}$ to week $5^{\textnormal{th}}$;
\item $1,100$ Euros per day from week $6^{\textnormal{th}}$ to week $28^{\textnormal{th}}$ (the increase in the daily budget was motivated to obtain more leads);
\item $700$ Euros per day from week $29^{\textnormal{th}}$ on (the decrease in the daily budget was motivated to reduce the cost per lead).
\end{itemize}
The algorithm was implemented in Python $2.7.12$ and executed on Ubuntu $16.04.1$ LTS with an Intel(R) Xeon(R) CPU E$5$-$2620$ v$3$ \@ $2.40$GHz. 
We used a discretization of the bid and daily budget space such that $|X| = 100$ and $|Y| = 500$. 
Furthermore, the estimates are based on the data collected in the last $20$ weeks, thus using a sliding window of $140$ days, to discard observations that were considered excessively old by human specialists.
The algorithm ran at midnight, collecting the observations of the day before, updating the models, computing the next values bid/daily budget to use, and, finally, setting these values on the corresponding platforms.
The total computing time of the algorithm per execution was shorter than $5$ minutes. 

\begin{figure}
	\centering
	\includegraphics[angle=90,width=\textwidth]{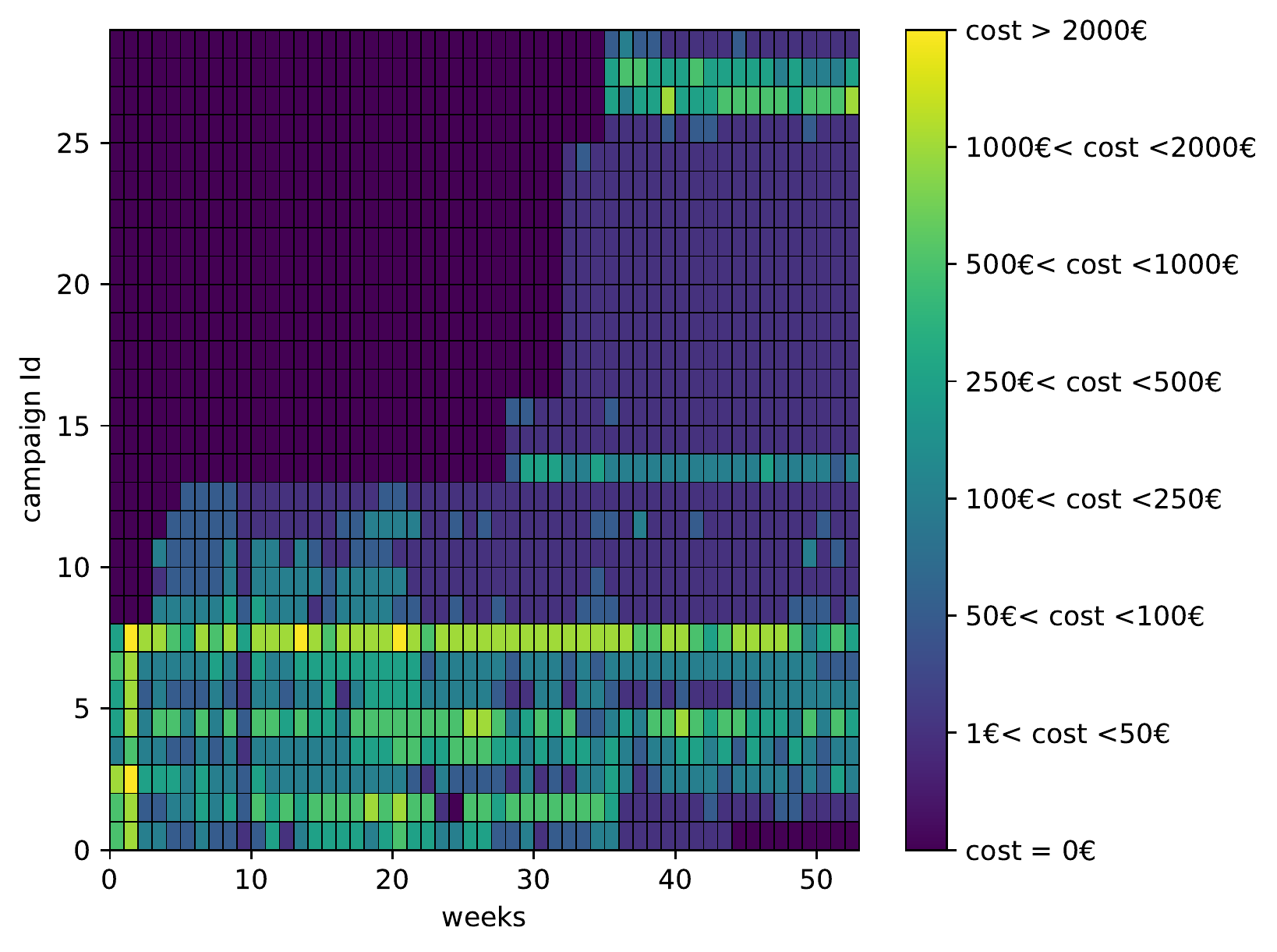}
	\caption{Actual costs for the campaigns. A cost of $0$ Euros means that the campaign is not active.}
	\label{fig:matrix}
\end{figure}

\begin{figure}[t!]
\begin{center}
\includegraphics{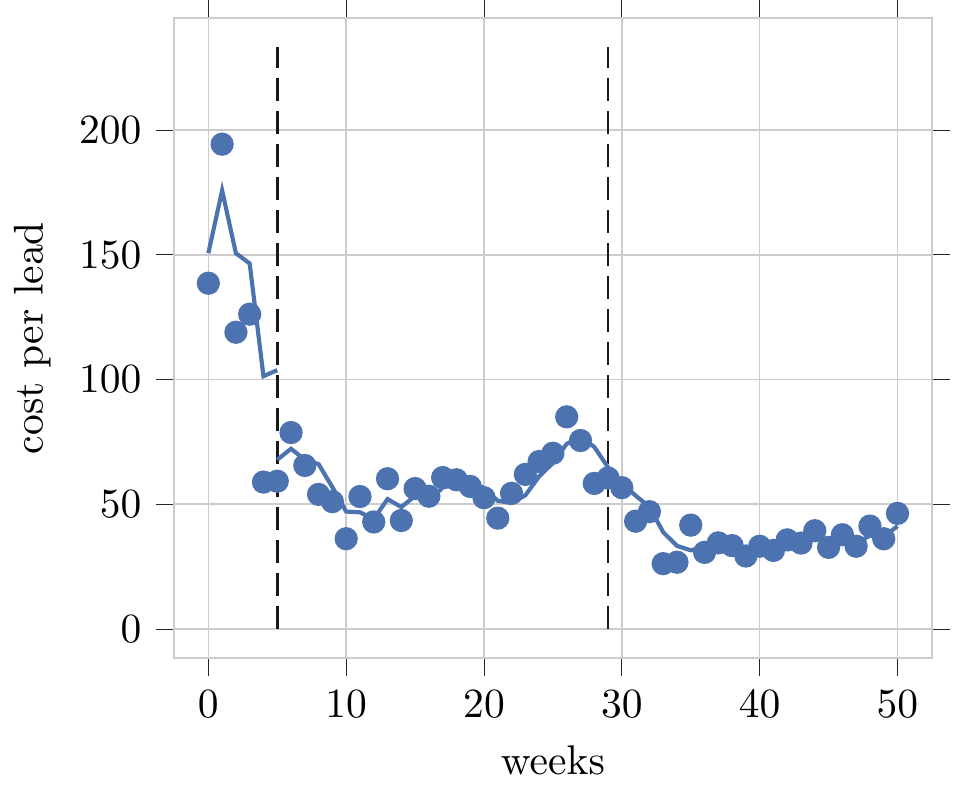}
\end{center}
\caption{Cost per lead across weeks.}
\label{fig:cost_per_conv}
\end{figure}

\begin{figure}[t!]
\begin{center}
\includegraphics{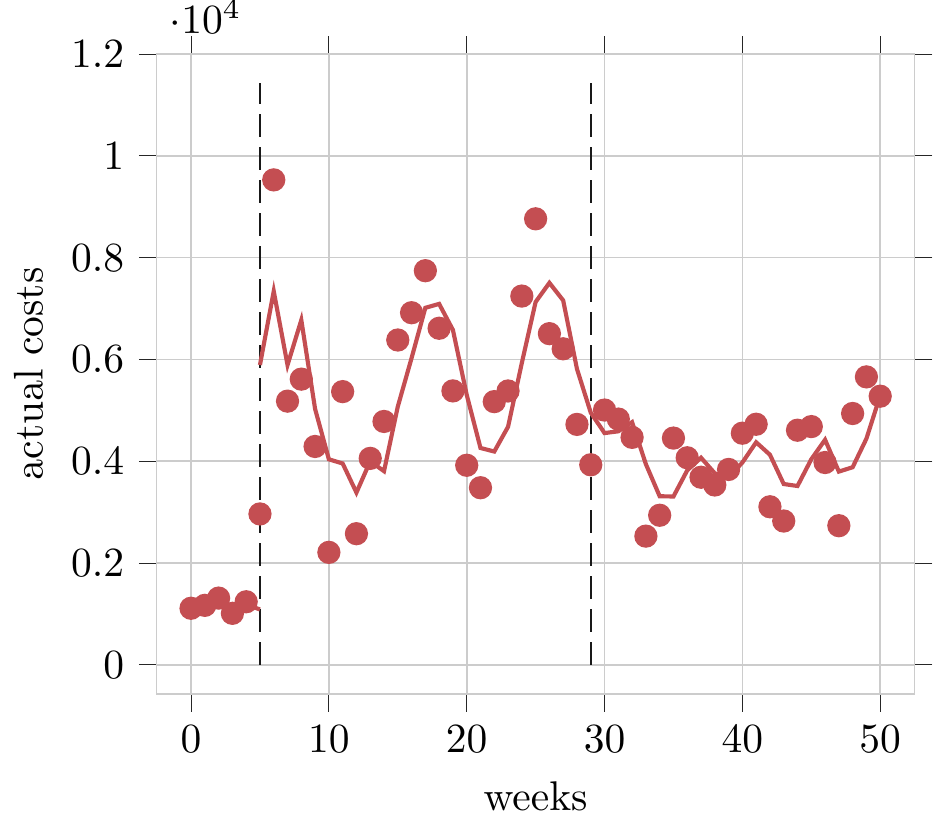}
\end{center}
\caption{Actual costs across weeks.}
\label{fig:costs}
\end{figure}

The cost per lead and the actual costs per week are reported in Figures~\ref{fig:cost_per_conv} and~\ref{fig:costs}, respectively.  
The cost per lead dramatically reduced thanks to the activation of algorithm at the $6^{\textnormal{th}}$ week, from an average of about $120$ Euros per lead to about $65$ Euros per lead.
Furthermore, the algorithm spent about $4$ weeks to reduce the cost per lead from about $65$ Euros (at the $6^{\textnormal{th}}$ week) to about $50$ Euros (at the $10^{\textnormal{th}}$ week).
The algorithm exhibited a rather explorative behavior until the $15^{\textnormal{th}}$ week due to the need to collect samples, whereas, from the $16^{\textnormal{th}}$ week to the $26^{\textnormal{th}}$ week, the algorithm presented a rather stable behavior.
Notably, even if the algorithm was rather explorative up to the $15^{\textnormal{th}}$ week, the performance in these weeks was evaluated rather stable by the human specialists.
The oscillations from the $15^{\textnormal{th}}$ week to the $28^{\textnormal{th}}$ week were due to seasonability effects and the campaigns of the competitors. 
More precisely, the human specialists confirmed that from the $12^{\textnormal{th}}$ week (beginning of September, corresponding to the conclusion of the summer holidays in Italy) to the $18^{\textnormal{th}}$ week (middle of October), the users are usually less interested to make loans. 
This behavior leaded to an increase of the cost per lead.
A similar effect was observed from the $22^{\text{th}}$ week to the $27^{\textnormal{th}}$ week (corresponding to the period from the beginning of December to the middle of January).
At the $29^{\textnormal{th}}$ week, the reduction of the daily budget pushed the algorithms to explore better the values of the functions to learn for smaller values of the daily budget.
This task required $4$ weeks, in which the cost per lead reduced from about $65$ Euros to about $35$ Euros.

The actual cost per week, differently from the cost per lead, was subject to prominent oscillations.
These oscillations were due to seasonability and the possibility of overspending on the platform (\emph{i.e.}, even if the platforms allow the introduction of daily budget constraints, these constraints can be violated by the platforms).  
The evaluation of the performance of our algorithm was very positive for the company, as to motivate its adoption for other sets of campaigns. 

\section{Conclusions and Future Works}
\label{sec:conclusions}

In the current paper, we present the \textsf{AdComB} algorithm, a method capable of choosing automatically the values of the bid and the daily budget of a set of Internet advertising campaigns to maximize, in an online fashion, the revenue under a budget constraint.
The algorithm exploits Gaussian Processes to estimate the campaigns performance, combinatorial bandit techniques to address the exploration/exploitation dilemma in the bid/daily budget choice, and dynamic programming techniques to solve the allocation optimization problem.
We propose four flavors of our algorithms: \textsf{AdComB-U-UCB} which uses an unfactorized model for the bid/daily budget space and confidence bounds for exploration, \textsf{AdComB-U-TS} which uses an unfactorized model and sampling as exploration strategy, \textsf{AdComB-F-UCB} which uses a factorized model and upper confidence bounds, and \textsf{AdComB-F-TS} which uses a factorized model and sampling for exploration. 
We theoretically analyze our algorithms and we provide high probability bounds on the regret of $\tilde{\mathcal{O}}(\sqrt{T})$, where $T$ is the time horizon of the learning process.
Our experimental results, on both synthetic settings and real-world settings, show that our algorithms tackle the problem properly, outperforming other naive algorithms based on existing solutions and the human expert, respectively.

As future work, we plan to study an algorithm for the adaptive discretization of the bid and budget space, depending on the complexity of the setting and the time horizon $T$.
Furthermore, while in the present work we assume that the environment, including the users and the other advertisers, is stationary over time, we will investigate non-stationary environments, \emph{e.g.}, including in the model the seasonality over different time scales of the user behaviour, as well as sudden changes due the modification of the competitors marketing policy.

\section*{Acknowledgments}

We would like to thank \emph{Mediamatic}, part of the Marketing Multimedia group, for supporting this research.

\clearpage
\bibliography{references}
\clearpage
\appendix
\section{Proofs}
\label{sex::appendix}

\begin{restatable}[From~\citep{srinivas2010gaussian}]{lem}{gauss} \label{lem:gauss}
	Given the realization of a GP $f(\cdot)$, the estimates of the mean $\hat{\mu}_{t-1}(x)$ and variance $\hat{\sigma}^2_{t-1}(x)$ for the input $x$ belonging to the input space $X$, for each $b \in \mathbb{R}^+$ the following condition holds:
	$$\mathbb{P} \left( |f(x) - \hat{\mu}_{t-1}(x)| \geq \sqrt{b} \ \hat{\sigma}_{t-1}(x) \right) \leq e^{-\frac{b}{2}},$$
	for each $x \in X$.
\end{restatable}

\begin{proof}
Be $r \sim \mathcal{N}(0,1)$ and $c \in \mathbb{R}^+$, we have:
\begin{equation*}
\mathbb{P} [ r > c ] = \frac{1}{\sqrt{2 \pi}} e^{-\frac{c^2}{2}} \int_{c}^{\infty} e^{-\frac{(r - c)^{2}}{2} - c(r-c)} \mathop{dr} \leq e^{- \frac{c^2}{2}} \mathbb{P} [ r > 0 ] = \frac{1}{2}e^{- \frac{c^2}{2}},
\end{equation*}
since $e^{-c (r-c) } \leq 1$ for $r \geq c$.
For the symmetry of the Gaussian distribution, we have:
\begin{equation*}
\mathbb{P} [ |r| > c ] \leq e^{- \frac{c^2}{2}}.
\end{equation*}
Applying the above result to $r = \frac{f(x) - \hat{\mu}_{t-1}(x)}{\hat{\sigma}_{t-1}(x)}$ and $c = \sqrt{b}$ concludes the proof.
\end{proof}

\regretducb*

\begin{proof}
In \alg{}-\textsf{U-UCB}, we assume the number of clicks $n_j(\bi, \bu) = n_j(\bm{a})$ of a campaign $C_j$ be the realization of a GP over the space $\mathcal{D}$ of the bid/daily budget pairs $\bm{a} = (\bi, \bu)$.
Using the selected input $\bm{a}_{j, h}$ and the corresponding observations $\tilde{n}_{j, h} = \tilde{n}_{j, h}(\bm{a}_{j, h})$ for each $h \in \{1 ,\ldots, t-1\}$, the GP provides the estimates of the mean $\hat{\mu}_{j,t-1}(\bm{a})$ and variance $\hat{\sigma}^2_{j,t-1}(\bm{a})$ for each $\bm{a} \in \mathcal{D}$.
The sampling phase is based on the upper bounds on the number of clicks and on the value per click, formally:
\begin{align}
	& u^{(n)}_{j,t-1}(\mathbf{a}) := \hat{\mu}_{j,t-1}(\mathbf{a}) + \sqrt{b_t} \ \hat{\sigma}_{j,t-1}(\mathbf{a}), \label{eq:upperclicku}\\
	& u^{(v)}_{j.t-1} := \hat{\nu}_{j, t-1} + \sqrt{b'_t} \ \hat{\psi}_{j,t-1} \label{eq:uppervalueu}.
\end{align}

Applying Lemma~\ref{lem:gauss} to Equation~\eqref{eq:upperclicku} for a generic arm $\mathbf{a}$ and $b = b_{t}$ we have:
\begin{equation*}
	\mathbb{P} \left[ |n_j(\bm{a}) - \mu_{j, t-1}(\bm{a})| > \sqrt{b_{t}} \hat{\sigma}_{j, t-1}(\bm{a}) \right] \leq e^{- \frac{b_{t}}{2}}.
\end{equation*}

In the execution of the \alg{}-\textsf{U-UCB} algorithm, after $t-1$ rounds, each arm can be chosen a number of times from $0$ to $t-1$.
Applying the union bound over the rounds ($t \in \{1, \ldots, T\}$), the campaigns ($j \in \{1, \ldots, N\}$) and the available arms in each campaign $\mathcal{D}$ ($\bm{a} \in \mathcal{D}$), and exploiting Lemma~\eqref{lem:gauss}, we obtain:
\begin{align*}
	&\mathbb{P} \left[ \bigcup_{t \in \{1, \ldots, T\} } \bigcup_{j \in \{1, \ldots, N\} } \bigcup_{\bm{a} \in \mathcal{D}} \left( | n_j(\bm{a}) - \hat{\mu}_{j,t-1}(\bm{a})| > \sqrt{b_{t}} \ \hat{\sigma}_{j,t-1}(\bm{a}) \right) \right]\\
	& \leq \sum_{t = 1}^T \sum_{j = 1}^N M e^{- \frac{b_{t}}{2}}.
\end{align*}
Thus, choosing $b_{t} = 2 \log{ \left( \frac{\pi^2 N M t^2}{3 \delta} \right)}$, we obtain:
\begin{align*}
	& \sum_{t = 1}^T \sum_{j = 1}^N M e^{- \frac{b_{t}}{2}} = \sum_{j = 1}^N \sum_{t=1}^{T} M \frac{3 \delta}{\pi^2 N M t^2} \leq \frac{\delta}{2} \frac{1}{N} \sum_{j = 1}^N \left(\frac{6}{\pi^2} \sum_{t=1}^\infty \frac{1}{t^2}\right) = \frac{\delta}{2}.
\end{align*}

Similarly, using Lemma~\ref{lem:gauss}, we have:
\begin{equation*}
	\mathbb{P} \left[ |v_j - \hat{\nu}_{j, t-1}| > \sqrt{b'_{t}} \ \hat{\psi}_{j,t-1} \right] \leq e^{- \frac{b'_{t}}{2} },
\end{equation*}
which holds for each $j \in \{1, \ldots, N\}$.
Choosing $b'_{t} = 2 \log{ \left( \frac{\pi^2 N t^2}{3 \delta} \right)}$ and applying an union bound we have:
\begin{align*}
	& \mathbb{P} \left[ \bigcup_{t \in \{1, \ldots, T\} } \bigcup_{j \in \{1, \ldots, N\} } \left( | v_j - \hat{\nu}_{j,t-1} | > \sqrt{b'_{t}} \ \hat{\psi}_{j,t-1} \right) \right]\\
	& \leq \sum_{j = 1}^N \sum_{t=1}^{T} e^{- \frac{b_{t}}{2}} = \sum_{j = 1}^N \sum_{t=1}^{T} \frac{3\delta}{\pi^2 N t^2} \leq \frac{\delta}{2}.
\end{align*}
Therefore, the event that at least one of the upper bounds over the number of clicks and the value per click does not hold has probability less than $\delta$.

Assume to be in the event that all the previous bounds hold.
The instantaneous pseudo-regret $reg_t$ at round $t$ satisfies the following inequality:
\begin{align*}
	reg_t &= r^*_{\bm{\mu}} - r_{\bm{\mu}}(S_t) \leq  r^*_{\bm{\mu}} - r_{\bar{\bm{\mu}}_{t}}(S_t) + r_{\bar{\bm{\mu}}_{t}}(S_t) - r_{\bm{\mu}}(S_t),
\end{align*}
where $\bar{\bm{\mu}}_t :=( u^{(v)}_{1, t-1} u^{(n)}_{1,t-1}(\mathbf{a}_1), \ldots, u^{(v)}_{N, t-1} u^{(n)}_{N,t-1}(\mathbf{a}_{M}))$ is the vector composed of all the upper bounds of the different arms (of dimension $N M$).
Let us recall that, given a generic superarm $S$, if all the elements of a vector $\bm{\mu}$ are larger than the ones of $\bm{\mu}'$ the following holds:
\begin{equation*}
	r_{\bm{\mu}}(S) \geq r_{\bm{\mu'}}(S).
\end{equation*}
Let us focus on the term $r_{\bar{\bm{\mu}}_t}(S_t)$. The following inequality holds:
\begin{align}
	r_{\bar{\bm{\mu}}_t}(S_t) & \geq \, r^*_{\bar{\bm{\mu}}_t} \geq  \,r_{\bar{\bm{\mu}}_t}(S^{*}_{\bm{\mu}}) \geq \, r_{\bm{\mu}}(S^{*}_{\bm{\mu}}) = \, r^*_{\bm{\mu}}, \label{eq:monoapprox}
\end{align}
where $S^*_{\bm{\mu}} \in \arg \max_{S \in \mathcal{S}}(r_{\bm{\mu}}(S))$ is the super-arm providing the optimum expected reward when the expected rewards are $\bm{\mu}$.
Thus, we have:
\begin{align*}
	reg_t & \leq r_{\bar{\bm{\mu}}_t}(S_t) - r_{\bm{\mu}}(S_t)\\
	& \leq r_{\bar{\bm{\mu}}_t}(S_t) - r_{\bm{\mu}_t}(S_t) + r_{\bm{\mu}_t}(S_t) - r_{\bm{\mu}}(S_t),
\end{align*}
where $\bm{\mu}_t := ( \hat{\eta}_{1, t-1} \hat{\mu}_{1,t-1}(\mathbf{a}_1), \ldots, \hat{\eta}_{N, t-1} \hat{\mu}_{N,t-1}(\mathbf{a}_{M}))$ is the vector composed of the estimated average payoffs for each arm $\bm{a} \in \mathcal{D}$.

A bound the terms $(r_{\bar{\bm{\mu}}_t}(S_t) - r_{\bm{\mu}_t}(S_t))$ is provided by the definition of the upper confidence bounds:
\begin{align*}
&r_{\bar{\bm{\mu}}_t} (S_t) - r_{\bm{\mu}_t}(S_t) = \sum_{j=1}^{N} \left[ u^{(v)}_{j,t-1} u^{(n)}_{j,t-1}(\mathbf{a}_{j,t}) - \hat{\nu}_{j,t-1} \hat{\mu}_{j,t-1}(\mathbf{a}_{j,t})\right]\\
&= \sum_{j=1}^{N} \left[ \hat{\nu}_{j,t-1} \sqrt{b_{t}} \ \hat{\sigma}_{j,t-1}(\mathbf{a}_{j,t}) + \hat{\mu}_{j,t-1}(\mathbf{a}_{j,t}) \sqrt{b'_{t}} \ \hat{\psi}_{j,t-1} + \sqrt{b_{t}} \ \hat{\sigma}_{j,t-1}(\mathbf{a}_{j,t}) \sqrt{b'_{t}} \ \hat{\psi}_{j,t-1} \right]\\
& \leq \sum_{j=1}^{N} \bigg\{ \left[ v_j + \sqrt{b'_{t}} \ \hat{\psi}_{j,t-1} \right] \sqrt{b_{t}} \ \hat{\sigma}_{j,t-1}(\mathbf{a}_{j,t}) \\
& + \left[n_j(\mathbf{a}_{j,t}) + \sqrt{b_{t}} \ \hat{\sigma}_{j,t-1}(\mathbf{a}_{j,t}) \right] \sqrt{b'_{t}} \ \hat{\psi}_{j,t-1} + \sqrt{b_{t}} \ \hat{\sigma}_{j,t-1}(\mathbf{a}_{j,t}) \sqrt{b'_{t}} \ \hat{\phi}_{j,t-1} \bigg\}\\
& \leq \sum_{j = 1}^N \bigg[ v_{\max} \sqrt{b_{t}} \max_{\bm{a} \in \mathcal{D}} \hat{\sigma}_{j,t-1}(\bm{a}) + n_{\max} \sqrt{b'_{t}} \ \hat{\psi}_{j,t-1} + 3 \sqrt{b_{t} b'_{t}} \ \hat{\psi}_{j,t-1} \max_{\bm{a} \in \mathcal{D}} \hat{\sigma}_{j,t-1}(\bm{a}) \bigg]\\
& \leq v_{\max} \sqrt{b_t} \sum_{j = 1}^N \max_{\bm{a} \in \mathcal{D}} \hat{\sigma}_{j,t-1}(\bm{a}) + n_{\max} \sqrt{b_t} \sum_{j = 1}^N \hat{\psi}_{j,t-1} + 3 \sqrt{b_t b'_t} \sum_{j = 1}^N \hat{\psi}_{j,t-1} \max_{\bm{a} \in \mathcal{D}} \hat{\sigma}_{j,t-1}(\bm{a})\\
& \leq v_{\max} \sqrt{b_t} \sum_{j = 1}^N \max_{\bm{a} \in \mathcal{D}} \hat{\sigma}_{j,t-1}(\bm{a}) + (n_{\max} \sqrt{b_t} + 3 \sqrt{b_t b'_t} \sigma) \sum_{j = 1}^N \hat{\psi}_{j,t-1},
\end{align*}
where $\mathbf{a}_{j,t}$ is the arm chosen for campaign $C_j$ in the superarm $S_t$, $v_{\max} := \max_{j \in \{1, \ldots, N \}} v_j$ is the maximum expected value per click, $n_{\max} := \max_{j, \mathbf{a}} n_j(\mathbf{a})$ is the maximum expected number of clicks for any campaign.
In the above derivation we used that $\sigma^2_{j,t}(\bm{a}) \leq k(\bm{a}, \bm{a}) =: \sigma^2$ for each $j$, $t$ and $\bm{a}$.


Let us focus on the term  $(r_{\bm{\mu}_t}(S_t) - r_{\bm{\mu}}(S_t))$:
\begin{align*}
&r_{\bm{\mu}_t}(S_t) - r_{\bm{\mu}}(S_t) = \sum_{j=1}^{N} \left[ \hat{\nu}_{j,t-1} \hat{\mu}_{j,t-1}(\mathbf{a}_{j,t}) - v_j n_j(\mathbf{a}_{j,t}) \right]\\
& = \sum_{j=1}^{N} \left[ \hat{\nu}_{j,t-1} \hat{\mu}_{j,t-1}(\mathbf{a}_{j,t}) - \hat{\nu}_{j,t-1} n_j(\mathbf{a}_{j,t}) + \hat{\nu}_{j,t-1} n_j(\mathbf{a}_{j,t}) - v_j n_j(\mathbf{a}_{j,t}) \right]\\
& \leq \sum_{j=1}^{N} \left[ (v_j + \sqrt{b'_{t}} \ \hat{\psi}_{j,t-1}) (\hat{\mu}_{j,t-1}(\mathbf{a}_{j,t}) - n_j(\mathbf{a}_{j,t})) + n_j(\mathbf{a}_{j,t}) (\hat{\nu}_{j,t-1} - v_j ) \right]\\
& \leq \sum_{j=1}^{N} (v_{\max} + \sqrt{b'_t} \hat{\psi}_{j,t-1}) \sqrt{b_t} \max_{\bm{a} \in \mathcal{D}} \hat{\sigma}_{j,t-1}(\bm{a}) + n_{\max} \sqrt{b'_t} \sum_{j=1}^{N} \hat{\psi}_{j,t-1}\\
& \leq v_{\max} \sqrt{b_t} \sum_{j=1}^{N} \max_{\bm{a} \in \mathcal{D}} \hat{\sigma}_{j,t-1}(\bm{a}) + (n_{\max} \sqrt{b_t} + \sqrt{b_t b'_t} \sigma) \sum_{j=1}^{N} \hat{\psi}_{j,t-1},
\end{align*}
where we used arguments similar to the ones considered in the previous derivation.

Overall, summing up the two terms, we have:
\begin{align*}
& reg_t  \leq  v_{\max} \sqrt{b_t} \sum_{j = 1}^N \max_{\bm{a} \in \mathcal{D}} \hat{\sigma}_{j,t-1}(\bm{a}) + (n_{\max} \sqrt{b_t} + 3 \sqrt{b_t b'_t} \sigma) \sum_{j = 1}^N \hat{\psi}_{j,t-1}\\
& + v_{\max} \sqrt{b_t} \sum_{j=1}^{N} \max_{\bm{a} \in \mathcal{D}} \hat{\sigma}_{j,t-1}(\bm{a}) + (n_{\max} \sqrt{b_t} + \sqrt{b_t b'_t} \sigma) \sum_{j=1}^{N} \hat{\psi}_{j,t-1}\\
& = 2 \sqrt{b_t} \left[ v_{\max} \sum_{j = 1}^N \max_{\bm{a} \in \mathcal{D}} \hat{\sigma}_{j,t-1}(\bm{a}) + (n_{\max} + 2 \sqrt{b'_t} \sigma) \sum_{j = 1}^N \hat{\psi}_{j,t-1} \right].
\end{align*}

We need now to upper bound $\hat{\sigma}_{i,t-1}(\bm{a})$ and $\hat{\psi}_{j,t-1}$.
Recall that, thanks to Lemma~$5.3$ in~\cite{srinivas2010gaussian}, under the Gaussian assumption we can express the information gain provided by the observations $\bm{n}_{t-1} = (\tilde{n}_{j,1}, \ldots, \tilde{n}_{j,t-1})$ corresponding to the sequence of arms $(\bm{a}_{j,1}, \ldots, \bm{a}_{j,t-1})$ as:
\begin{equation*}
IG(\bm{n}_{t-1} \,|\, n_j) = \frac{1}{2} \sum_{h = 1}^{t-1} \log \left( 1 + \frac{\hat{\sigma}^{2}_{j,h}(\bm{a}_{j,h})}{\lambda} \right).
\end{equation*}
Since $b_{h}$ is non-decreasing in $h$, we can write:
\begin{align}
\sigma^{2}_{j, h}(\bm{a}_{j,h}) & = \lambda \left[ \frac{\hat{\sigma}^{2}_{j,h}(\bm{a}_{j,h})}{\lambda} \right] \leq \frac{\log \left( {1 + \frac{\hat{\sigma}^{2}_{j,h}(\bm{a}_{j,h})}{\lambda}} \right)} {\log \left( 1 + \frac{1}{\lambda} \right) }, \label{eq:boundvar}
\end{align}
since $s^{2} \leq \frac{\log{(1 + s^{2})}}{\lambda \log \left( 1 + \frac{1}{\lambda} \right)}$ for all $s \in [0, \frac{1}{\lambda}]$, and $\frac{\hat{\sigma}^{2}_{j,h}(\bm{a}_{j,h})}{\lambda} \leq \frac{k(\bm{a}_{j,h}, \bm{a}_{j,h})}{\lambda} \leq \frac{1}{\lambda}$.

Using the definition of $\hat{\psi}_{j,t-1}$ we have:
\begin{align*}
	\sum_{h=1}^{t-1} \hat{\psi}^2_{j,t-1} = \sum_{h=1}^{t-1} \frac{\psi_j^2 \xi}{\xi + (h-1)\psi_j^2} \leq \xi \log \left( \frac{\xi}{\psi_j^2} + t \right).
\end{align*}

Since Equation~\eqref{eq:boundvar} holds for any $\bm{a} \in \mathcal{D}$, then it also holds for the arm $\bm{a}_{\max}$ maximizing the variance $\sigma^{2}_{j,h}(\bm{a}_{j,h})$ in $n_j$ defined over $\mathcal{D}$.
Thus, using the Cauchy-Schwarz inequality, we obtain:

\begin{align*}
&\mathcal{R}_T^{2}(\mathfrak{U}) \leq T \sum_{t=1}^{T} reg^{2}_t \\
& \leq 4 T b_T \sum_{t=1}^{T} \left[ 2 v_{\max}^2 \left( \sum_{j = 1}^N \max_{\bm{a} \in \mathcal{D}_j} \sigma_{j,t-1}(\bm{a}) \right)^2 + 2 ( n_{\max} + 2 \sqrt{b'_t}\sigma)^2 \left( \sum_{j = 1}^N \hat{\psi}_{j,t-1} \right)^2 \right]\\
& \leq 8 T b_T \Bigg\{ \sum_{t=1}^{T} \left[ v_{\max}^2 N \sum_{j=1}^N \max_{\bm{a} \in \mathcal{D}} \hat{\sigma}^2_{j,t-1}(\bm{a}) \right] + \sum_{t=1}^{T} \left[ (n_{\max} + 2 \sqrt{b'_t} \sigma)^2 N \sum_{j = 1}^N \hat{\psi}^2_{j,t-1} \right] \Bigg\}\\
& \leq 8 T N b_T \Bigg\{ v_{\max}^2 \sum_{j=1}^N \sum_{t=1}^{T} \left[ \max_{\bm{a} \in \mathcal{D}} \frac{\log \left(1 + \frac{ \hat{\sigma}^{2}_{i,n-1}(\bm{a}) }{\lambda} \right) }{\log \left( 1 + \frac{1}{\lambda} \right)} \right]\\
& + (n_{\max} + 2 \sqrt{b'_t} \sigma)^2 \sum_{j=1}^N \sum_{t=1}^{T} \hat{\psi}^{2}_{j,t-1} \Bigg\}\\
& \leq 8 T N b_T \Bigg\{ \frac{v_{\max}^2}{\log \left( 1 + \frac{1}{\lambda} \right)} \sum_{j=1}^N \underbrace{ \sum_{t=1}^{T} \max_{\bm{a} \in \mathcal{D}} \log \left(1 + \frac{ \hat{\sigma}^{2}_{i,n-1}(\bm{a}) }{\lambda} \right) }_{= \gamma_T(n_j)}\\
& + \xi (n_{\max} + 2 \sqrt{b'_t} \sigma)^2 \sum_{j=1}^N \log \left( \frac{\xi}{\psi_j^2} + T \right) \Bigg\}\\
& \leq 8 T N b_T \left[ \frac{v_{\max}^2}{\log \left( 1 + \frac{1}{\lambda} \right)} \sum_{j=1}^N \gamma_T(n_j) + \xi (n_{\max} + 2 \sqrt{b'_t} \sigma)^2 \sum_{j=1}^N \log \left( \frac{\xi}{\psi_j^2} + T \right) \right].
\end{align*}
We conclude the proof by taking the square root on both the r.h.s.~and the l.h.s.~of the last inequality.
\end{proof}

\regretdts*

\begin{proof}
Recall that in \alg{}-\textsf{2D-TS} we assume the number of clicks $n_j(\bi, \bu) = n_j(\bm{a})$ of a campaign $C_j$ is the realization of a GP over the space $\mathcal{D}$ of the bid/budget pairs $\bm{a} = (\bi, \bu)$.
Using the selected input $\bm{a}_{j, h}$ and corresponding observations $\tilde{n}_{j, h} = \tilde{n}_{j, h}(\bm{a}_{j, h})$ for each $h \in \{1 ,\ldots, t-1\}$ the GP provides us with the estimates of the mean $\hat{\mu}_{j,t-1}(\bm{a})$ and variance $\hat{\sigma}^2_{j,t-1}(\bm{a})$ for each $\bm{a} \in \mathcal{D}$.
The sampling phase generates the following two values for the number of clicks and on the value per click, formally, for each campaign $C_j$, a sample $\theta^{(n)}_{j,t-1}(\mathbf{a})$ is extracted from $\mathcal{N}(\hat{\mu}_{j,t-1}(\mathbf{a}), \hat{\sigma}^2_{j,t-1}(\mathbf{a}))$ for the number of clicks, and a sample $\theta^{(v)}_{j,t-1}$ is extracted from $\mathcal{N}(\hat{\nu}_{j, t-1}, \hat{\psi}_{j,t-1})$.

Let us focus on $\theta^{(n)}_{j, t-1}(\mathbf{a})$.
Since Lemma~\ref{lem:gauss} also applies to univariate Gaussian distributions, it holds for $\theta^{(n)}_{j,t-1}(\mathbf{a})$, for a generic arm $\bm{a}$, and, formally, we have:
\begin{equation*}
	\mathbb{P} \left[ |\theta^{(n)}_{j, t-1}(\bm{a}) - \hat{\mu}_{j, t-1}(\bm{a})| > \sqrt{b_{t}} \hat{\sigma}_{j, t-1}(\bm{a}) \right] \leq e^{- \frac{b_{t}}{2}},
\end{equation*}
for each $b_{t} > 0$.
By relying on the triangle inequality, fora each $\bm{a} \in \mathcal{D}$ we have:
\begin{align*}
	& \mathbb{P} \left[ | \theta^{(n)}_{j,t-1}(\bm{a}) - n_j(\bm{a}) | > \sqrt{b_{t}} \hat{\sigma}_{j, t-1}(\bm{a}) \right]\\
	& \leq \mathbb{P} \left[ |\theta^{(n)}_{j,t-1}(\bm{a}) - \hat{\mu}_{j,t-1}(\bm{a})| + |\hat{\mu}_{j,t-1}(\bm{a}) - n_j(\bm{a})| > \sqrt{b_{t}} \hat{\sigma}_{j, t-1}(\bm{a}) \right] \\
	& \leq \mathbb{P} \left[ |\theta^{(n)}_{j,t-1}(\bm{a}) - \hat{\mu}_{j,t-1}(\bm{a})| > \frac{1}{2} \sqrt{b_{t}} \hat{\sigma}_{j, t-1}(\bm{a}) \right]\\
	& + \mathbb{P} \left[ |\hat{\mu}_{j,t-1}(\bm{a}) - n_j(\bm{a})| > \frac{1}{2} \sqrt{b_{t}} \hat{\sigma}_{j, t-1}(\bm{a}) \right] \\
	& \leq 2 e^{- \frac{b_{t}}{8}}.
\end{align*}
Similarly to what done in the proof of Theorem~\ref{thm:regretducb}, setting $b_{t} := 8 \log \left( \frac{2 N M t^2}{3 \delta} \right)$, applying the union bound over the rounds, the subsets $\mathcal{D}$, the number of times the arms are chosen in $\mathcal{D}$, and the available arms, we have that the following holds with probability at least $1 - \frac{\delta}{2}$:
\begin{equation*}
	|\theta^{(n)}_{j, h-1}(\bm{a}) - n_j(\bm{a})| < \sqrt{b_{h}} \hat{\sigma}_{j, h-1}(\bm{a}),
\end{equation*}
for all $\bm{a} \in \mathcal{D}_j$, $j \in \{1, \ldots N\}$ and $h \in \{1, \ldots, t\}$.

The same reasoning can be carried out with $\theta^{(v)}_{j,t-1}$ setting $b'_{t} := 8 \log \left( \frac{2 N t^2}{3 \delta} \right)$, so that the following bound:
\begin{equation*}
|\theta^{(v)}_{j, t-1} - v_j| < \sqrt{b'_{t}} \hat{\psi}_{j, t-1},
\end{equation*}
holds for each $j \in \{1, \ldots, N\}$ and $h \in \{1, \ldots, t\}$ with probability at least $1 - \frac{\delta}{2}$.
Therefore, jointly, the bounds over the number of clicks and on the value per click hold with probability at least $1 - \delta$.

Let us assume that all previous bounds hold.
Consider the instantaneous pseudo-regret $reg_t$ at round $t$:
\begin{align*}
reg_t &= r^*_{\bm{\mu}} - r_{\bm{\mu}}(S_t)\\
& =  r^*_{\bm{\mu}} - r_{\bm{\theta_t}}(S^*_{\bm{\mu}}) + r_{\bm{\theta_t}}(S^*_{\bm{\mu}}) - r_{\bm{\theta_t}}(S_t) + r_{\bm{\theta_t}}(S_t) - r_{\bm{\mu}}(S_t)\\
& \leq |r_{\bm{\mu}}(S^*_{\bm{\mu}}) - r_{\bm{\theta_t}}(S^*_{\bm{\mu}})| + |r_{\bm{\theta_t}}(S_t) - r_{\bm{\mu}}(S_t)|,
\end{align*}
where $\bm{\theta_t} := (\theta^{(v)}_{1,t-1} \theta^{(n)}_{1,t-1}(\bm{a}_1), \ldots, \theta^{(v)}_{N,t-1} \theta^{(n)}_{N,t-1}(\bm{a}_{M}))$ is the vector of the drawn payoffs for the turn $t$ and $r_{\bm{\theta_t}}(S^*_{\bm{\mu}}) - r_{\bm{\theta_t}}(S_t) \leq 0$ for the fact that the chosen arm $S_t$ maximize the reward assuming an expected reward over the arms of $\bm{\theta_t}$.

Let us focus on the term  $| r_{\bm{\mu}}(S) - r_{\bm{\theta_t}}(S) |$ on a generic superarm $S = (\bm{a}_1, \ldots \bm{a}_N)$:
\begin{align*}
&|r_{\bm{\mu}}(S) - r_{\bm{\theta_t}}(S)| = \sum_{j=1}^{N} | v_j n_j(\mathbf{a}_j) - \theta^{(v)}_{j,t-1} \theta^{(n)}_{j,t-1}(\mathbf{a}_j) |\\
&= \sum_{j=1}^{N} | v_j n_j(\mathbf{a}_j) - \theta^{(v)}_{j,t-1} n_j(\mathbf{a}_j) | + | \theta^{(v)}_{j,t-1} n_j(\mathbf{a}_j) - \theta^{(v)}_{j,t-1} \theta^{(n)}_{j,t-1}(\mathbf{a}_j) | \\
& = \sum_{j=1}^{N} \left[ n_{\max} \sqrt{b'_{t}} \hat{\psi}_{j, t-1} + \theta^{(v)}_{j,t-1} \sqrt{b_{t}} \hat{\sigma}_{j, t-1}(\bm{a}_j) \right]\\
& = n_{\max} \sqrt{b'_t} \sum_{j=1}^{N} \hat{\psi}_{j, t-1} + \sum_{j=1}^{N} \left( v_j + \sqrt{b'_{t}} \hat{\psi}_{j, t-1} \right) \sqrt{b_{t}} \hat{\sigma}_{j, t-1}(\bm{a}_j) \\
& = n_{\max} \sqrt{b'_t} \sum_{j=1}^{N} \hat{\psi}_{j, t-1} + v_{\max} \sqrt{b_t} \sum_{j=1}^{N} \hat{\sigma}_{j, t-1}(\bm{a}_j) + \sqrt{b_t b'_t} \sigma \sum_{j=1}^{N} \hat{\psi}_{j, t-1}\\
& \leq  v_{\max} \sqrt{b_t} \sum_{j=1}^{N} \max_{\bm{a} \in \mathcal{D}} \hat{\sigma}_{j,t-1}(\bm{a}) + \sqrt{b'_t}(n_{\max} + \sqrt{b_t} \sigma) \sum_{j=1}^{N} \hat{\psi}_{j, t-1},
\end{align*}
and, therefore, the instantaneous regret $reg_t$ is bounded by twice the quantity we derived.

The cumulative regret becomes:
\begin{align*}
&\mathcal{R}_T^{2}(\mathfrak{U}) \leq T \sum_{t=1}^{T} reg^{2}_t \\
& \leq 4 T \sum_{t=1}^{T} \left[ 2 v_{\max}^2 b_t \left( \sum_{j = 1}^N \max_{\bm{a} \in \mathcal{D}_j} \sigma_{j,t-1}(\bm{a}) \right)^2 + 2 b'_t ( n_{\max} + \sqrt{b_t} \sigma)^2 \left( \sum_{j = 1}^N \hat{\psi}_{j,t-1} \right)^2 \right]\\
& \leq 8 T \Bigg\{ b_T \sum_{t=1}^{T} \left[ v_{\max}^2 N \sum_{j=1}^N \max_{\bm{a} \in \mathcal{D}} \hat{\sigma}^2_{j,t-1}(\bm{a}) \right] + b'_T \sum_{t=1}^{T} \left[ (n_{\max} + \sqrt{b_t} \sigma)^2 N \sum_{j = 1}^N \hat{\psi}^2_{j,t-1} \right] \Bigg\}\\
& \leq 8 T N \left[ \frac{v_{\max}^2}{\log \left( 1 + \frac{1}{\lambda} \right)} b_T \sum_{j=1}^N \gamma_T(n_j) + \xi b'_T (n_{\max} + \sqrt{b_T} \sigma)^2 \sum_{j=1}^N \log \left( \frac{\xi}{\psi_j^2} + T \right) \right].
\end{align*}
We conclude the proof by taking the square root on both the r.h.s.~and the l.h.s.~of the last inequality.

\end{proof}
\clearpage
\regretfucb*

\begin{proof}
At first notice that Lemma~\ref{lem:gauss} can be applied to the quantities of maximum number of clicks, maximum cost and value per click.
This allows, setting $b_{t}:= 2 \log \left( \frac{\pi^2 N M t^2}{2 \delta} \right)$ and $b'_{t} := 2 \log \left( \frac{\pi^2 N t^2}{2 \delta} \right)$, that the following bounds for each arm $\bi$ and each round $t$ hold at the same time:
\begin{align*}
	&| n^{\mathsf{sat}}_j(\bi) - \hat{\mu}_{j,t-1}(\bi)| \leq \sqrt{b_{t}} \hat{\sigma}_{j,t-1}(\bi),\\
	&| e_{j}^{\mathsf{sat}}(\bi) - \hat{\eta}_{j,t-1}(\bi) | \leq \sqrt{b_{t}} \hat{s}_{j,t-1}(\bi),\\
	&|v_j - \hat{\nu}_{j,t-1}| \leq \sqrt{b'_{t}} \hat{\psi}_{j,t-1},
\end{align*}
with probability at least $1 - \delta$, since each one of the above events holds with probability at least $1 - \frac{\delta}{3}$.

Assume that the previous bounds hold and consider the following quantity:
\begin{equation} \label{eq:nfac}
	\left| \min \left\{n_{j}^{\mathsf{sat}}(\bi), \bu e^{\mathsf{sat}}_{j}(\bi) \right\} - \min \left\{\hat{\mu}_{j,t-1}(\bi), \bu \hat{\eta}_{j,t-1}(\bi) \right\} \right|.
\end{equation}
If we are able to provide a bound for this quantity, then following the proof of Theorem~\ref{thm:regretducb} it is possible to provide a bound on the regret of the \alg{}-\textsf{F-UCB} algorithm.

Depending on the values of $y$, $e^{\mathsf{sat}}_{j}(\bi)$ and $\hat{\eta}_{j,t-1}(\bi)$ we can distinguish the following $4$ cases:

\noindent \textbf{Case 1}: if $\bu c^{\mathsf{sat}}_{j}(\bi) > n^{\mathsf{sat}}_j(\bi) \wedge \bu \hat{\eta}_{j,t-1}(\bi) > \hat{\mu}_{j,t-1}(\bi)$ the quantity in Equation~\ref{eq:nfac} becomes:
\begin{align}
	&\left| \min \left\{n_{j}^{\mathsf{sat}}(\bi), \bu e^{\mathsf{sat}}_{j}(\bi) \right\} - \min \left\{\hat{\mu}_{j,t-1}(\bi), \bu \hat{\eta}_{j,t-1}(\bi) \right\} \right|\\
	&\leq |n_{j}^{\mathsf{sat}}(\bi) - \hat{\mu}_{j,t-1}(\bi)| \leq \sqrt{b_{t}} \hat{\sigma}_{j,t-1}(\bi).
\end{align}
	
\noindent \textbf{Case 2} $y c^{\mathsf{sat}}_{j}(\bi) < n^{\mathsf{sat}}_j(\bi) \wedge \bu \hat{\eta}_{j,t-1}(\bi) < \hat{\mu}_{j,t-1}(\bi)$ the quantity in Equation~\ref{eq:nfac} becomes:
\begin{align}
	&\left| \min \left\{n_{j}^{\mathsf{sat}}(\bi), \bu e^{\mathsf{sat}}_{j}(\bi) \right\} - \min \left\{\hat{\mu}_{j,t-1}(\bi), \bu \hat{\eta}_{j,t-1}(\bi) \right\} \right|\\
	& = \bu \left| e^{\mathsf{sat}}_{j}(\bi) - \hat{\eta}_{j,t-1}(\bi) \right| \leq \bu_{\max} \sqrt{b_{t}} \hat{s}_{j,t-1}(\bi).
\end{align}

\noindent \textbf{Case 3}: $ \frac{n^{\mathsf{sat}}_{j}(\bi)}{e^{\mathsf{sat}}_{j}(\bi)} < \bu < \frac{\hat{\mu}_{j,t-1}(\bi)}{\hat{\eta}_{j,t-1}(\bi)}$ the quantity in Equation~\ref{eq:nfac} becomes:
\begin{align}
	&\left| \min \left\{n_{j}^{\mathsf{sat}}(\bi), \bu e^{\mathsf{sat}}_{j}(\bi) \right\} - \min \left\{\hat{\mu}_{j,t-1}(\bi), \bu \hat{\eta}_{j,t-1}(\bi) \right\} \right|\\
	& = \left| n_{j}^{\mathsf{sat}}(\bi) - \bu \hat{\eta}_{j,t-1}(\bi) \right|\\
	& \leq \bu \left| e^{\mathsf{sat}}_{j}(\bi) - \hat{\eta}_{j,t-1}(\bi) \right| \leq \bu_{\max} \sqrt{b_{t}} \hat{s}_{j,t-1}(\bi),
\end{align}
where we used that $n^{\mathsf{sat}}_{j}(\bi) \leq \bu e^{\mathsf{sat}}_{j}(\bi)$.

\noindent \textbf{Case 4}: $\frac{\hat{\mu}_{j,t-1}(\bi)}{\hat{\eta}_{j,t-1}(\bi)} < \bu < \frac{n^{\mathsf{sat}}_{j}(\bi)}{e^{\mathsf{sat}}_{j}(\bi)}$ the quantity in Equation~\ref{eq:nfac} becomes:
\begin{align}
	&\left| \min \left\{n_{j}^{\mathsf{sat}}(\bi), \bu e^{\mathsf{sat}}_{j}(\bi) \right\} - \min \left\{\hat{\mu}_{j,t-1}(\bi), \bu \hat{\eta}_{j,t-1}(\bi) \right\} \right|\\
	& = \left| \bu e^{\mathsf{sat}}_{j}(\bi) - \hat{\mu}_{j,t-1}(\bi) \right|\\
	& \leq | n_{j}^{\mathsf{sat}}(\bi) - \hat{\mu}_{j,t-1}(\bi)| \leq \sqrt{b_{t}} \hat{\sigma}_{j,t-1}(\bi),
\end{align}
where we used that $n^{\mathsf{sat}}_{j}(\bi) \geq \bu e^{\mathsf{sat}}_{j}(\bi)$.

Overall we have that:
\begin{align}
	& \left| \min \left\{n_{j}^{\mathsf{sat}}(\bi), \bu e^{\mathsf{sat}}_{j}(\bi) \right\} - \min \left\{\hat{\mu}_{j,t-1}(\bi), \bu \hat{\eta}_{j,t-1}(\bi) \right\} \right|\\
	& \leq \sqrt{b_t} \left( \bu_{\max} \hat{s}_{j,t-1}(\bi) +  \hat{\sigma}_{j,t-1}(\bi) \right).
\end{align}

Similarly to what has been provided above, let us focus on the quantity:
\begin{equation} \label{eq:nfac2}
	\left| \min \left\{ u^{(n)}_{j,t-1}(\bi), \bu u^{(e)}_{j,t-1}(\bi) \right\} -  \min \left\{ \hat{\mu}_{j,t-1}(\bi), \bu \hat{\eta}_{1,t-1}(\bi) \right\} \right|,
\end{equation}
which can be bounded by looking at the same $4$ different cases.

\noindent \textbf{Case 1}: if $\bu u^{(e)}_{j,t-1}(\bi) > u^{(n)}_{j,t-1}(\bi) \wedge \bu \hat{\eta}_{j,t-1}(\bi) > \hat{\mu}_{j,t-1}(\bi)$ the quantity in Equation~\ref{eq:nfac2} becomes by definition:
\begin{align}
	&\left| \min \left\{ u^{(n)}_{j,t-1}(\bi), \bu u^{(e)}_{j,t-1}(\bi) \right\} -  \min \left\{ \hat{\mu}_{j,t-1}(\bi), \bu \hat{\eta}_{1,t-1}(\bi) \right\} \right|\\
	& = | u^{(n)}_{j,t-1}(\bi) - \hat{\mu}_{j,t-1}(\bi)| = \sqrt{b_{t}} \hat{\sigma}_{j,t-1}(\bi).
\end{align}

\noindent \textbf{Case 2} $\bu u^{(e)}_{j,t-1}(\bi) < u^{(n)}_{j,t-1}(\bi) \ \wedge \ \bu \hat{\eta}_{j,t-1}(\bi) < \hat{\mu}_{j,t-1}(\bi)$ the quantity in Equation~\ref{eq:nfac2} becomes:
\begin{align}
	&\left| \min \left\{ u^{(n)}_{j,t-1}(\bi), \bu u^{(e)}_{j,t-1}(\bi) \right\} -  \min \left\{ \hat{\mu}_{j,t-1}(\bi), \bu \hat{\eta}_{j,t-1}(\bi) \right\} \right|\\
	& = \bu \left| u^{(e)}_{j,t-1}(\bi) - \hat{\eta}_{1,t-1}(\bi) \right|\\
	& \leq \bu_{\max} \sqrt{b_{t}} \hat{s}_{j,t-1}(\bi).
\end{align}

\noindent \textbf{Case 3}:  $ \frac{u^{(n)}_{j,t-1}(\bi)}{u^{(e)}_{j,t-1}(\bi)} < \bu < \frac{\hat{\mu}_{j,t-1}(\bi)}{\hat{\eta}_{j,t-1}(\bi)}$ the quantity in Equation~\ref{eq:nfac2} becomes:
\begin{align}
	&\left| \min \left\{ u^{(n)}_{j,t-1}(\bi), \bu u^{(e)}_{j,t-1}(\bi) \right\} -  \min \left\{ \hat{\mu}_{j,t-1}(\bi), \bu \hat{\eta}_{1,t-1}(\bi) \right\} \right|\\
	& = \left| u^{(n)}_{j,t-1}(\bi) - \bu \hat{\eta}_{j,t-1}(\bi) \right|\\
	& \leq \bu \left| u^{(e)}_{j,t-1}(\bi) - \hat{\eta}_{j,t-1}(\bi) \right|\\
	& \leq \bu_{\max} \sqrt{b_{t}} \hat{s}_{j,t-1}(\bi).
\end{align}

\textbf{Case 4}: $\frac{\hat{\mu}_{j,t-1}(\bi)}{\hat{\eta}_{j,t-1}(\bi)} < \bu < \frac{u^{(n)}_{j,t-1}(\bi)}{u^{(e)}_{j,t-1}(\bi)}$ the quantity in Equation~\ref{eq:nfac2} becomes:
\begin{align}
	&\left| \min \left\{ u^{(n)}_{j,t-1}(\bi), \bu u^{(e)}_{j,t-1}(\bi) \right\} -  \min \left\{ \hat{\mu}_{j,t-1}(\bi), \bu \hat{\eta}_{1,t-1}(\bi) \right\} \right|\\
	& = \left| \bu u^{(e)}_{j,t-1}(\bi) - \hat{\mu}_{j,t-1}(\bi) \right|\\
	& \leq \left| u^{(n)}_{j,t-1}(\bi) - \hat{\mu}_{j,t-1}(\bi) \right| = \sqrt{b_{t}} \hat{\sigma}_{j,t-1}(\bi).
\end{align}

Overall we have the bound:
\begin{align}
	& \left| \min \left\{ u^{(n)}_{j,t-1}(\bi), \bu u^{(e)}_{j,t-1}(\bi) \right\} - \min \left\{ \hat{\mu}_{j,t-1}(\bi), \bu \hat{\eta}_{1,t-1}(\bi) \right\} \right|\\
	& \leq \sqrt{b_t} \left( \bu_{\max} \hat{s}_{j,t-1}(\bi) +  \hat{\sigma}_{j,t-1}(\bi) \right).
\end{align}

Since the definition of per round regret $reg_t$ is the same as the one in Theorem~\ref{thm:regretducb}, defining:
\begin{align*}
	\bar{\bm{\mu}}_t := \Bigg( u^{(v)}_{1, t-1} & \min \left\{ u^{(n)}_{1,t-1}(x_1), \bu_1 u^{(e)}_{1,t-1}(x_1) \right\}, \ldots,\\
	& u^{(v)}_{N, t-1} \min \left\{ u^{(n)}_{N,t-1}(x_{M}), \bu_{M} u^{(e)}_{N, t-1}(x_{M}) \right\} \Bigg),\\
	\bm{\mu}_t := \Bigg( \hat{\nu}_{1,t-1} & \min \left\{ \hat{\mu}_{1,t-1}(x_1), \bu_1 \hat{\eta}_{1,t-1}(x_1) \right\}, \ldots, \\
	&\hat{\nu}_{N, t-1} \min \left\{ \hat{\mu}_{N,t-1}(x_{M}), \bu_{M} \hat{\eta}_{N, t-1}(x_{M}) \right\} \Bigg)\\
	\bm{\mu} := \Bigg( v_1 \ & \min \left\{ n_1^{\mathsf{sat}}(x_1), \bu_1 c_1^{\mathsf{sat}}(x_1) \right\}, \ldots, \\
	&v_N \ \min \left\{ n_N^{\mathsf{sat}}(x_{M}), \bu_{M} c_N^{\mathsf{sat}}(x_{M}) \right\} \Bigg),
\end{align*}
we have:
\begin{align*}
	reg_t & \leq r_{\bar{\bm{\mu}}_t}(S_t) - r_{\bm{\mu}}(S_t)\\
	& \leq r_{\bar{\bm{\mu}}_t}(S_t) - r_{\bm{\mu}_t}(S_t) + r_{\bm{\mu}_t}(S_t) - r_{\bm{\mu}}(S_t),
\end{align*}

A bound the terms $(r_{\bar{\bm{\mu}}_t}(S_t) - r_{\bm{\mu}_t}(S_t))$ is provided by the definition of the upper confidence bounds:
\begin{align*}
&r_{\bar{\bm{\mu}}_t} (S_t) - r_{\bm{\mu}_t}(S_t)\\
& = \sum_{j=1}^{N} \left[ u^{(v)}_{j,t-1} \min \left\{ u^{(n)}_{j,t-1}(x_{j,t}), \bu_{j,t} u^{(e)}_{j,t-1}(x_{j,t}) \right\} - \hat{\nu}_{j,t-1} \min \left\{ \hat{\mu}_{j,t-1}(x_{j,t}), \bu_{j,t} \hat{\eta}_{j,t-1}(x_{j,t}) \right\} \right]\\
& = \sum_{j=1}^{N} \left[ u^{(v)}_{j,t-1} \min \left\{ u^{(n)}_{j,t-1}(x_{j,t}), \bu_{j,t} u^{(e)}_{j,t-1}(x_{j,t}) \right\} - u^{(v)}_{j,t-1} \min \left\{ \hat{\mu}_{j,t-1}(x_{j,t}), \bu_{j,t} \hat{\eta}_{j,t-1}(x_{j,t}) \right\} \right.\\
& \left. + u^{(v)}_{j,t-1} \min \left\{ \hat{\mu}_{j,t-1}(x_{j,t}), \bu_{j,t} \hat{\eta}_{j,t-1}(x_{j,t}) \right\} - \hat{\nu}_{j,t-1} \min \left\{ \hat{\mu}_{j,t-1}(x_{j,t}), \bu_{j,t} \hat{\eta}_{j,t-1}(x_{j,t}) \right\} \right]\\
& \leq \sum_{j=1}^{N} \Bigg\{ \left[ v_{\max} + 2 \sqrt{b'_{t}} \hat{\psi}_{j,t-1} \right] \Big( \min \left\{ u^{(n)}_{j,t-1}(x_{j,t}), \bu_{j,t} u^{(e)}_{j,t-1}(x_{j,t}) \right\} \\
& - \min \left\{ \hat{\mu}_{j,t-1}(x_{j,t}), \bu_{j,t} \hat{\eta}_{j,t-1}(x_{j,t}) \right\} \Big)\\
& + (u^{(v)}_{j,t-1} - \hat{\nu}_{j,t-1}) \left[ n_{\max}^{\mathsf{sat}} + \sqrt{b_t} \left( \bu_{\max} \hat{s}_{j,t-1}(\bi) + \hat{\sigma}_{j,t-1}(\bi) \right) \right] \Bigg\}\\
& \leq \sum_{j=1}^{N} \left\{ \left[ v_{\max} + 2 \sqrt{b'_{t}} \hat{\psi}_{j,t-1} \right] \sqrt{b_t} \left( \bu_{\max} \hat{s}_{j,t-1}(\bi) + \hat{\sigma}_{j,t-1}(\bi) \right) \right.\\
& + \left. \left( n_{\max}^{\mathsf{sat}} + \sqrt{b_t} \left( \bu_{\max} \hat{s}_{j,t-1}(\bi) + \hat{\sigma}_{j,t-1}(\bi) \right) \right) \sqrt{b'_{t}} \hat{\psi}_{j,t-1} \right\}\\
& = v_{\max} \ \bu_{\max} \sqrt{b_t} \sum_{j=1}^{N} \hat{s}_{j,t-1}(\bi) + v_{\max} \sqrt{b_t} \sum_{j=1}^{N} \hat{\sigma}_{j,t-1}(\bi) + 2 \sqrt{b_t b'_t} (s \bu_{\max} + \sigma) \sum_{j=1}^{N} \hat{\psi}_{j,t-1}\\
& + n_{\max}^{\mathsf{sat}} \sqrt{b'_t} \sum_{j=1}^{N} \hat{\psi}_{j,t-1} + s \bu_{\max} \sqrt{b_t b'_t} \sum_{j=1}^{N} \hat{\psi}_{j,t-1} + \sigma \sqrt{b_t b'_t} \sum_{j=1}^{N} \hat{\psi}_{j,t-1}\\
& = v_{\max} \ \bu_{\max} \sqrt{b_t} \sum_{j=1}^{N} \hat{s}_{j,t-1}(\bi) + v_{\max} \sqrt{b_t} \sum_{j=1}^{N} \hat{\sigma}_{j,t-1}(\bi) + \bigg( 3 s \bu_{\max} \sqrt{b_t b'_t} \\
& + n_{\max}^{\mathsf{sat}} \sqrt{b'_t} + 3 \sigma \sqrt{b_t b'_t} \bigg) \sum_{j=1}^{N} \hat{\psi}_{j,t-1},
\end{align*}
where $v_{\max} := \max_{j=1}^N v_j$ is the maximum expected value per click, $n_{\max} := max_{j, x} \ n_j(x)$ is the maximum number of clicks for any campaign, and we have $\hat{\sigma}_{j,t}(x) \leq \sigma$ and $\hat{s}_{j,t}(x) \leq s$ for each $j$, $t$, and $x$.

Let us focus on the term $(r_{\bm{\mu}_t}(S_t) - r_{\bm{\mu}}(S_t))$:
\begin{align*}
&r_{\bm{\mu}_t}(S_t) - r_{\bm{\mu}}(S_t)\\
& = \sum_{j=1}^{N} \left[ \hat{\nu}_{j,t-1} \min \left\{ \hat{\mu}_{j,t-1}(x_{j,t}), \bu_{j,t} \hat{\eta}_{j,t-1}(x_{j,t}) \right\}  - v_j \min \left\{ n_j^{\mathsf{sat}}(x_{j,t}), \bu_{j,t} e_j^{\mathsf{sat}}(x_{j,t}) \right\} \right]\\
& = \sum_{j=1}^{N} \left[ \hat{\nu}_{j,t-1} \min \left\{ \hat{\mu}_{j,t-1}(x_{j,t}), \bu_{j,t} \hat{\eta}_{j,t-1}(x_{j,t}) \right\} - \hat{\nu}_{j,t-1} \min \left\{ n_j^{\mathsf{sat}}(x_{j,t}), \bu_{j,t} e_j^{\mathsf{sat}}(x_{j,t}) \right\} \right.\\
& \left. + \hat{\nu}_{j,t-1} \min \left\{ n_j^{\mathsf{sat}}(x_{j,t}), \bu_{j,t} e_j^{\mathsf{sat}}(x_{j,t}) \right\} - v_j \min \left\{ n_j^{\mathsf{sat}}(x_{j,t}), \bu_{j,t} e_j^{\mathsf{sat}}(x_{j,t}) \right\} \right]\\
& \leq \sum_{j=1}^{N} \left\{ \left[ v_{\max} + \sqrt{b_{t}} \hat{\psi}_{j,t-1} \right] \Big( 
\min \left\{ \hat{\mu}_{j,t-1}(x_{j,t}), \bu_{j,t} \hat{\eta}_{j,t-1}(x_{j,t}) \right\} \right. \\
& \left. - \min \left\{ n_j^{\mathsf{sat}}(x_{j,t}), \bu_{j,t} e_j^{\mathsf{sat}}(x_{j,t}) \right\} \Big) + (\hat{\nu}_{j,t-1} - v_{j,t-1}) n_{\max}^{\mathsf{sat}} \right\}\\
& \leq \sum_{j=1}^{N} \left\{ \left[ v_{\max} + \sqrt{b'_{t}} \hat{\psi}_{j,t-1} \right] \sqrt{b_t} \left( \bu_{\max} \hat{s}_{j,t-1}(\bi) + \hat{\sigma}_{j,t-1}(\bi) \right) + n_{\max}^{\mathsf{sat}} \sqrt{b'_{t}} \hat{\psi}_{j,t-1} \right\}\\
& = v_{\max} \ \bu_{\max} \sqrt{b_t} \sum_{j=1}^{N} \hat{s}_{j,t-1}(\bi) +
v_{\max} \sqrt{b_t} \sum_{j=1}^{N} \hat{\sigma}_{j,t-1}(\bi)\\
& + ( \sqrt{b_t b'_t} \sigma + s \bu_{\max} \sqrt{b_t b'_t} + n_{\max}^{\mathsf{sat}}  \sqrt{b'_t}) \sum_{j=1}^{N} \hat{\psi}_{j,t-1} \\
\end{align*}
where we used arguments similar to the ones we considered in the previous derivation.

Summing up we have:
\begin{align*}
reg_t & \leq 2 v_{\max} \sqrt{b_t} \sum_{j=1}^{N} \hat{\sigma}_{j,t-1}(\bi) + 2 v_{\max} \ \bu_{\max} \sqrt{b_t} \sum_{j=1}^{N} \hat{s}_{j,t-1}(\bi)\\
& + \sqrt{b'_t} \left(4 s \bu_{\max} \sqrt{b_t} + 4 \sigma \sqrt{b_t} + 2 n_{\max}^{\mathsf{sat}} \right)  \sum_{j=1}^{N} \hat{\psi}_{j,t-1},
\end{align*}

Using arguments similar to what has been used in Theorem~\ref{thm:regretducb} we have:
\begin{align*}
&\mathcal{R}_T^{2}(\mathfrak{U}) \leq T \sum_{t=1}^{T} reg^{2}_t \\
& = T \sum_{t=1}^{T} \left[ 2 v_{\max} \sqrt{b_t} \sum_{j=1}^{N} \hat{\sigma}_{j,t-1}(\bi) + 2 v_{\max} \ \bu_{\max} \sqrt{b_t} \sum_{j=1}^{N} \hat{s}_{j,t-1}(\bi)\right.\\
&\left.  + \sqrt{b'_t} \left(4 s \bu_{\max} \sqrt{b_t} + 4 \sigma \sqrt{B_t} + 2 n_{\max}^{\mathsf{sat}} \right) \sum_{j=1}^{N} \hat{\psi}_{j,t-1} \right]^2\\
& \leq T \left[ 12 v^2_{\max} b_T N \sum_{j=1}^{N} \sum_{t=1}^{T} \max_{\bi} \hat{\sigma}^2_{j,t-1}(\bi) + 12 v^2_{\max} \ \bu^2_{\max} b_T N \sum_{j=1}^{N} \sum_{t=1}^{T} \max_{\bi} \hat{s}^2_{j,t-1}(\bi) \right.\\
&\left. + 3 b'_T N \left(4 s \bu_{\max} \sqrt{b_T} + 4 \sigma \sqrt{b_T} + 2 n_{\max}^{\mathsf{sat}} \right)^2 \sum_{j=1}^{N} \sum_{t=1}^{T} \hat{\psi}^2_{j,t-1} \right]^2\\
& T N \left[ \bar{c}_1 b_T \sum_{j=1}^{N} \gamma_T(n_j) + \bar{c}_2 b_T \sum_{j=1}^{N} \gamma_T(e_j) \right.\\
&\left. + \bar{c}_3 b'_T \left(2 s \bu_{\max} \sqrt{b_T} + 2 \sigma \sqrt{b_T} + n_{\max}^{\mathsf{sat}} \right)^2 \sum_{j=1}^{N} \log \left( \frac{\xi}{\psi^2_j} + T \right) \right],
\end{align*}
where we defined $\bar{c}_1 := \frac{12 v^2_{\max}}{\log \left( 1 + \frac{1}{\lambda} \right)}$, $\bar{c}_2 := \frac{12 v^2_{\max} \bu^2_{\max}}{\log \left( 1 + \frac{1}{\lambda'} \right)}$, and $\bar{c}_3 := 12 \xi$, where $\lambda$ and $\lambda'$ are the variance of the measurement noise on the maximum number of clicks and number of clicks per unit of daily budget.
Taking the square root of both right and left hand side of this inequality concludes the proof.

\end{proof}

\regretfts*
\begin{proof}
We recall that the decision we take are based on samples drawn from these distributions:
\begin{align*}
	&\theta^{(n)}_{j,t-1}(\bi) \sim \mathcal{N}(\hat{\mu}_{j,t-1}(\bi),\hat{\sigma}^2_{j,t-1}(\bi)),\\
	&\theta^{(e)}_{j,t-1}(\bi) \sim \mathcal{N}(\hat{\eta}_{j,t-1}(\bi), \hat{s}^2_{j,t-1}(\bi)),\\
	&\theta^{(v)}_{j,t-1} \sim \mathcal{N}(\hat{\nu}_{j,t-1}, \hat{\psi}^2_{j,t-1}).
\end{align*}

Ad we derived in the proof of Theorem~\ref{thm:regretdts}, by extracting samples from the predictive distributions of the GP we have that:
\begin{align*}
	& \mathbb{P} \left[ |\theta^{(n)}_{j,t}(\bi) - n^{\mathsf{sat}}_j(\bi)| > \sqrt{b_{t}} \hat{\sigma}_{j, t-1}(\bi) \right] \leq 2 e^{- \frac{b_{t}}{8}},\\
	& \mathbb{P} \left[ |\theta^{(e)}_{j,t}(\bi) - e^{\mathsf{sat}}_j(\bi)| > \sqrt{b_{t}} \hat{s}_{j, t-1}(\bi) \right] \leq 2 e^{- \frac{b_{t}}{8}},\\
	& \mathbb{P} \left[ |\theta^{(v)}_{j,t} - v_j| > \sqrt{b'_{t}} \hat{\psi}_{j, t-1} \right] \leq 2 e^{- \frac{b'_{t}}{8}},
\end{align*}
and, therefore, setting $b_{t} := 8 \log \left( \frac{2 N M t^2}{2 \delta} \right)$ and $b'_{t} := 8 \log \left( \frac{2 N t^2}{2 \delta} \right)$ we have that the bounds over the rounds $t$, the different GPs $j$ and the different arms $\bi$ holds together with probability greater than $1 - \delta$.

Using arguments similar to what has been used in Theorem~\ref{thm:regretfucb}, we obtain the following bound for a generic bid/budget pair $(\bi,\bu)$:
\begin{align*}
	& |\min \{ n_j^{\mathsf{sat}}(\bi), \bu e_j^{\mathsf{sat}}(\bi) \} - \min \{\theta^{(n)}_{j,t}(\bi_j), \bu_j \theta^{(e)}_{j,t}(\bi_j) \}| \\
	& \leq |\min \{ n_j^{\mathsf{sat}}(\bi), \bu e_j^{\mathsf{sat}}(\bi) \} - \min \left\{ \hat{\mu}_{j,t-1}(\bi), \bu \hat{\eta}_{1,t-1}(\bi) \right\}|\\
	& + |\min \left\{ \hat{\mu}_{j,t-1}(\bi), \bu \hat{\eta}_{1,t-1}(\bi) \right\} - \min \{\theta^{(n)}_{j,t}(\bi_j), \bu_j \theta^{(e)}_{j,t}(\bi_j) \}|\\
	& \leq 2 \sqrt{b_t} \left( \bu_{\max} \hat{s}_{j,t-1}(\bi) +  \hat{\sigma}_{j,t-1}(\bi) \right).
\end{align*}

As in Theorem~\ref{thm:regretdts}, the instantaneous pseudo-regret $reg_t$ at round $t$ is bounded as follows:
\begin{align*}
reg_t \leq |r_{\bm{\mu}}(S^*_{\bm{\mu}}) - r_{\bm{\theta_t}}(S^*_{\bm{\mu}})| + |r_{\bm{\theta_t}}(S_t) - r_{\bm{\mu}}(S_t)|,
\end{align*}
where $\bm{\mu}$ us defined as in Theorem~\ref{thm:regretfucb} and 
\begin{align*}
	\bm{\theta_t} := (\theta^{(v)}_{1,t} \min \{\theta^{(n)}_{1,t}(\bi_1), & \bu_1 \theta^{(e)}_{1,t}(\bi_1) \}, \ldots,\\ 
	& \theta^{(v)}_{N,t} \min \{\theta^{(n)}_{N,t}(\bi_{M}), \bu_ {M} \theta^{(e)}_{1,t}(\bi_{M}) \}
\end{align*}
is the vector of the drawn payoffs for the turn $t$ and $r_{\bm{\theta_t}}(S^*_{\bm{\mu}}) - r_{\bm{\theta_t}}(S_t) \leq 0$ for the fact that the chosen arm $S_t$ maximize the reward assuming an expected reward over the arms of $\bm{\theta_t}$.

Let us focus on bounding the quantity $|r_{\bm{\mu}}(S) - r_{\bm{\theta_t}}(S)|$ on a generic superarm $S$:
\begin{align*}
	&|r_{\bm{\mu}}(S) - r_{\bm{\theta_t}}(S)| = \\
	& = \sum_{j=1}^{N} \left[v_j \min \{ n_j^{\mathsf{sat}}(x_j), \bu_j e_j^{\mathsf{sat}}(x_j) \} - \theta^{(v)}_{j,t} \min \{\theta^{(n)}_{j,t}(\bi_j), \bu_j \theta^{(e)}_{j,t}(\bi_j) \} \right]\\
	& = \sum_{j=1}^{N} \left[ v_j \min \{ n_j^{\mathsf{sat}}(x_j), \bu_j e_j^{\mathsf{sat}}(x_j) \}
	- v_j \min \{\theta^{(n)}_{j,t}(\bi_j), \bu_j \theta^{(e)}_{j,t}(\bi_j) \} \right.\\
	&\left. + v_j \min \{\theta^{(n)}_{j,t}(\bi_j), \bu_j \theta^{(e)}_{j,t}(\bi_j) \} - \theta^{(v)}_{j,t} \min \{\theta^{(n)}_{j,t}(\bi_j), \bu_j \theta^{(e)}_{j,t}(\bi_j) \} \right]\\
	& \leq \sum_{j=1}^{N} \left[ 2 v_{\max} \sqrt{b_t} \left( \bu_{\max} \hat{s}_{j,t-1}(\bi) + \hat{\sigma}_{j,t-1}(\bi) \right) \right.\\
	& \left. + \sqrt{b'_{t}} \hat{\psi}_{j, t-1} \left( n_{\max}^{\mathsf{sat}} + 2 \sqrt{b_t} \left( \bu_{\max} \hat{s}_{j,t-1}(\bi) + \hat{\sigma}_{j,t-1}(\bi) \right) \right) \right]\\
	& \leq 2 v_{\max} \bu_{\max} \sqrt{b_t} \sum_{j=1}^{N} \hat{s}_{j,t-1}(\bi) + 2 v_{\max} \sqrt{b_t} \sum_{j=1}^{N} \hat{\sigma}_{j,t-1}(\bi) + \\
	& + \sqrt{b'_t} ( n_{\max}^{\mathsf{sat}} + 2 s \bu_{\max} \sqrt{b_t} + 2 \sigma \sqrt{b_t}) \sum_{j=1}^{N} \hat{\psi}_{j, t-1})
\end{align*}

Thus, the the instantaneous pseudo-regret $reg_t$ at round $t$ is bounded as follows:
\begin{align*}
reg_t & \leq 4 v_{\max} \bu_{\max} \sqrt{b_t} \sum_{j=1}^{N} \hat{s}_{j,t-1}(\bi) + 4 v_{\max} \sqrt{b_t} \sum_{j=1}^{N} \hat{\sigma}_{j,t-1}(\bi) + \\
& + 2 \sqrt{b'_t} ( n_{\max}^{\mathsf{sat}} + 2 s \bu_{\max} \sqrt{b_t} + 2 \sigma \sqrt{b_t}) \sum_{j=1}^{N} \hat{\psi}_{j, t-1})
\end{align*}

Defining the constants $\bar{c}_1 := \frac{48 v^2_{\max}}{\log \left( 1 + \frac{1}{\lambda} \right)}$, $\bar{c}_2 := \frac{48 v^2_{\max} \bu^2_{\max}}{\log \left( 1 + \frac{1}{\lambda'} \right)}$, and $\bar{c}_3 := 12 \xi$ we have a cumulative regret of:
\begin{align*}
	&\mathcal{R}_T^{2}(\mathfrak{U}) \leq T \sum_{t=1}^{T} reg^{2}_t \\
	&T N \left[ \bar{c}_1 b_T \sum_{j=1}^{N} \gamma_T(n_j) + \bar{c}_2 b_T \sum_{j=1}^{N} \gamma_T(e_j) \right.\\
	&\left. + \bar{c}_3 b'_T \left(2 s \bu_{\max} \sqrt{b_T} + 2 \sigma \sqrt{b_T} + n_{\max}^{\mathsf{sat}} \right)^2 \sum_{j=1}^{N} \log \left( \frac{\xi}{\psi^2_j} + T \right) \right],
\end{align*}
and taking the square root of both the left and right hand side of the inequality concludes the proof.

\end{proof}
\clearpage

\end{document}